\documentclass[journal,twocolumn,final]{IEEEtran}
\usepackage{cite}

\ifCLASSINFOpdf
\else
\fi

\usepackage{graphicx}
\usepackage{wrapfig}
\usepackage{color}
\usepackage{tikz}
\usetikzlibrary{decorations.pathreplacing,angles,quotes}

\usepackage{amsmath}
\usepackage{amssymb}
\usepackage{algorithm2e}
\usepackage{amsthm}

\DeclareMathOperator*{\argmin}{arg\,min}
\DeclareMathOperator*{\argmax}{arg\,max}

\newtheorem{theorem}{Theorem}
\usetikzlibrary{external}
\begin{document}
\title{Alignment Distances on Systems of Bags}

\author{Alexander Sagel and Martin Kleinsteuber% <-this % stops a space
\thanks{Copyright \copyright\ 2017 IEEE. Personal use of this material is permitted. However, permission to use this material for any other purposes must be obtained from the IEEE by sending an email to pubs-permissions@ieee.org. 

The authors are with the Department of Electrical and Computer Engineering at the Technical University of Munich, Arcisstr. 21, 80333 M\"unchen, Germany. M. Kleinsteuber is also with Mercateo AG, F\"urstenfelder Str. 5, 80331 M\"unchen“, e-mail: a.sagel@tum.de, martin.kleinsteuber@mercateo.com}
}% <-this % stops a space}
%\author{anonymous submission}

\markboth{Preprint}%
{Shell \MakeLowercase{\textit{et al.}}: Bare Demo of IEEEtran.cls for IEEE Journals}



\maketitle

\begin{abstract}
Recent research in image and video recognition indicates that many visual processes can be thought of as being generated by a time-varying generative model. A nearby descriptive model for visual processes is thus a statistical distribution that varies over time. Specifically, modeling visual processes as streams of histograms generated by a kernelized linear dynamic system turns out to be efficient.  We refer to such a model as a System of Bags. In this work, we investigate Systems of Bags with special emphasis on dynamic scenes and dynamic textures. Parameters of linear dynamic systems suffer from ambiguities. In order to cope with these ambiguities in the kernelized setting, we develop a kernelized version of the alignment distance. For its computation, we use a Jacobi-type method and prove its convergence to a set of critical points. We employ it as a dissimilarity measure on Systems of Bags. As such, it outperforms other known dissimilarity measures for kernelized linear dynamic systems, in particular the Martin Distance and the Maximum Singular Value Distance, in every tested classification setting. A considerable margin can be observed in settings, where classification is performed with respect to an abstract mean of video sets. For this scenario, the presented approach can outperform state-of-the-art techniques, such as Dynamic Fractal Spectrum or Orthogonal Tensor Dictionary Learning.
\end{abstract}


% \fi
\IEEEpeerreviewmaketitle

\section{Introduction}
\IEEEPARstart{M}{any} of the most successful classification frameworks for videos employ generative models of visual processes, where videos are modeled as distributions of descriptors. Prominent examples are \emph{Local Binary Patterns in Three Orthogonal Plains} (LBP-TOP) \cite{zhao2007dynamic} and \emph{Bags of Systems} (BoS) \cite{mumtaz2015scalable,mumtaz2013clustering,ravichandran2013categorizing}. Typically, the descriptors in question are \emph{local} in the spatiotemporal domain, but the distributions are \emph{global}, neglecting their spatial and temporal order. This has proven successful in many classification problems, however, there are scenarios where such a procedure could turn out suboptimal for several reasons. Spatiotemporally local descriptors are supposed to capture the dynamics locally in space and time. This is sensible for the recognition of dynamic textures on a small scale, but for large-scale dynamic textures or real-world dynamic scenes, the temporal dynamics on a global scale is a more distinguishing feature: for instance, a traffic scene is more characterized by the appearance and disappearance of vehicles than by the movement of the trees on the roadside. Furthermore, breaking up the global temporal order of the overall visual process can be problematic in cases where the appearance of the frames changes as a whole over the course of time, e.g. when observing outdoor scenes under changing weather conditions. In such cases, the video can have semantic features that may get lost by destroying the temporal order.  On the other hand, breaking up the global spatial order contradicts the everyday observation that looking at single frames of a visual process often suffices to distinguish between dynamic scenes or high-resolution textures. In such cases, employing well established still image feature extraction methods on the isolated frames can be a sensible step in the feature extraction of the overall visual process.

Remarkably, for still image textures and still image scenes, generative, distribution-based models have proven their efficacy at several occasions. Texture images, being often thought of as realizations of stochastic processes \cite{mallat2016understanding}, have a long standing tradition of distribution based models \cite{haralick1979statistical,unser1986sum,ojala1996comparative,do2002wavelet}. Meanwhile, the concept of BoS is inspired by the \emph{Bag of Words} (BoW) \cite{fei2005bayesian} paradigm, where images are described by the frequency of previously learned features contained in them. BoW based methods have been successfully  employed in the task of distinguishing still-image scenes \cite{juneja2013blocks,zhou2013scene}. Moreover, outstanding  performance on dynamic scene recognition can be achieved when "bags" of spatiotemporal features are computed on a temporally local scale. For instance, the authors of \cite{feichtenhofer2014bags} propose computing several temporally localized bags of oriented filter response features from  videos and produce outstanding results. The classification is performed by a majority vote that encompasses all of the computed bags in a video.

We conclude that, employing generative, distribution based models for the individual frames, or alternatively, localized collections of frames of visual processes such as dynamic textures and dynamic scenes is a promising approach. In the classical case, including many BoW based approaches, these models are histograms. Alternatively, they can also be represented by \emph{Fisher Vectors} \cite{perronnin2010improving} or statistical moments of a parametrized distribution model \cite{do2002wavelet}. We will focus on the classical view in the following, even though the proposed methods can be easily generalized to the other perspectives. This leads to the assumption that individual frames of a visual process can be well modeled by histograms. However, a model treating visual processes as sequences of histograms neglects their temporal dynamics and thus fails to generalize from one sample to the whole process.

A remedy is to derive a dynamic model for the temporal evolution of the histograms.
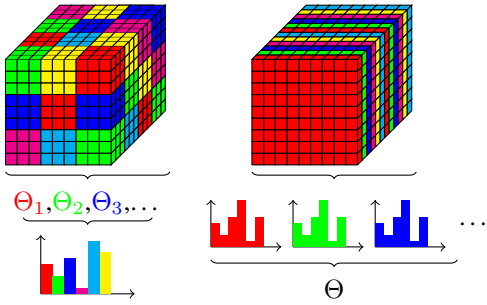
\begin{figure}
	\begin{center}
		
	\begin{tikzpicture}[scale=0.78]
		\foreach \y in {1,2,3} {  
			\foreach \x in {1,2,3} {  
				\filldraw[fill=green] (1.8,\x/5,\y/5) -- (1.8,\x/5,\y/5+0.2) -- (1.8,\x/5+0.2,\y/5+0.2) -- (1.8,\x/5+1/5,\y/5) -- (1.8,\x/5,\y/5);
	        }
        }
        \foreach \y in {1,2,3} {  
        	\foreach \x in {1,2,3} {  
        		\filldraw[fill=magenta] (1.8,0.6+\x/5,\y/5) -- (1.8,0.6+\x/5,\y/5+0.2) -- (1.8,0.6+\x/5+0.2,\y/5+0.2) -- (1.8,0.6+\x/5+1/5,\y/5) -- (1.8,0.6+\x/5,\y/5);
        	}
        }
        \foreach \y in {1,2,3} {  
           	\foreach \x in {1,2,3} {  
               	\filldraw[fill=blue] (1.8,1.2+\x/5,\y/5) -- (1.8,1.2+\x/5,\y/5+0.2) -- (1.8,1.2+\x/5+0.2,\y/5+0.2) -- (1.8,1.2+\x/5+1/5,\y/5) -- (1.8,1.2+\x/5,\y/5);
           	}
        }
		\foreach \y in {1,2,3} {  
			\foreach \x in {1,2,3} {  
				\filldraw[fill=red] (1.8,\x/5,\y/5+0.6) -- (1.8,\x/5,\y/5+0.8) -- (1.8,\x/5+0.2,\y/5+0.8) -- (1.8,\x/5+1/5,\y/5+0.6) -- (1.8,\x/5,\y/5+0.6);
			}
		}
     	\foreach \y in {1,2,3} {  
     		\foreach \x in {1,2,3} {  
     			\filldraw[fill=yellow] (1.8,0.6+\x/5,\y/5+0.6) -- (1.8,0.6+\x/5,\y/5+0.8) -- (1.8,0.6+\x/5+0.2,\y/5+0.8) -- (1.8,0.6+\x/5+1/5,\y/5+0.6) -- (1.8,0.6+\x/5,\y/5+0.6);
     			}
     		}
     	\foreach \y in {1,2,3} {

     		\foreach \x in {1,2,3} {  
     			\filldraw[fill=magenta] (1.8,1.2+\x/5,\y/5+0.6) -- (1.8,1.2+\x/5,\y/5+0.2+0.6) -- (1.8,1.2+\x/5+0.2,\y/5+0.2+0.6) -- (1.8,1.2+\x/5+1/5,\y/5+0.6) -- (1.8,1.2+\x/5,\y/5+0.6);
      			}
       	}
		\foreach \y in {1,2,3} {  
			\foreach \x in {1,2,3} {  
				\filldraw[fill=cyan] (1.8,\x/5,\y/5+1.2) -- (1.8,\x/5,\y/5+1.4) -- (1.8,\x/5+0.2,\y/5+1.4) -- (1.8,\x/5+1/5,\y/5+1.2) -- (1.8,\x/5,\y/5+1.2);
			}
		}
		\foreach \y in {1,2,3} {  
			\foreach \x in {1,2,3} {  
				\filldraw[fill=green] (1.8,0.6+\x/5,\y/5+1.2) -- (1.8,0.6+\x/5,\y/5+1.4) -- (1.8,0.6+\x/5+0.2,\y/5+1.4) -- (1.8,0.6+\x/5+1/5,\y/5+1.2) -- (1.8,0.6+\x/5,\y/5+1.2);
			}
		}
		\foreach \y in {1,2,3} {  
			\foreach \x in {1,2,3} {  
				\filldraw[fill=yellow] (1.8,1.2+\x/5,\y/5+1.2) -- (1.8,1.2+\x/5,\y/5+1.4) -- (1.8,1.2+\x/5+0.2,\y/5+1.4) -- (1.8,1.2+\x/5+1/5,\y/5+1.2) -- (1.8,1.2+\x/5,\y/5+1.2);
			}
		}
		\foreach \y in {1,2,3} {  
			\foreach \x in {1,2,3} {  
				\filldraw[fill=green] (1.8,\x/5,\y/5+1.8) -- (1.8,\x/5,\y/5+2) -- (1.8,\x/5+0.2,\y/5+2) -- (1.8,\x/5+1/5,\y/5+1.8) -- (1.8,\x/5,\y/5+1.8);
			}
		}
		\foreach \y in {1,2,3} {  
			\foreach \x in {1,2,3} {  
				\filldraw[fill=blue] (1.8,0.6+\x/5,\y/5+1.8) -- (1.8,0.6+\x/5,\y/5+2) -- (1.8,0.6+\x/5+0.2,\y/5+2) -- (1.8,0.6+\x/5+1/5,\y/5+1.8) -- (1.8,0.6+\x/5,\y/5+1.8);
			}
		}
		\foreach \y in {1,2,3} {  
			\foreach \x in {1,2,3} {  
				\filldraw[fill=red] (1.8,1.2+\x/5,\y/5+1.8) -- (1.8,1.2+\x/5,\y/5+0.2+1.8) -- (1.8,1.2+\x/5+0.2,\y/5+0.2+1.8) -- (1.8,1.2+\x/5+1/5,\y/5+1.8) -- (1.8,1.2+\x/5,\y/5+1.8);
			}
		}
	
		\foreach \y in {1,2,3} {  
			\foreach \x in {1,2,3} {  
				\filldraw[fill=red] (1+\x/5,\y/5+1.2,2.6) -- (1+\x/5,\y/5+1.4,2.6) -- (1.2+\x/5,\y/5+1.4,2.6) -- (1.2+\x/5,\y/5+1.2,2.6)  -- (1+\x/5,\y/5+1.2,2.6);
			}
		}
		\foreach \y in {1,2,3} {  
			\foreach \x in {1,2,3} {  
				\filldraw[fill=blue] (1+\x/5,\y/5+0.6,2.6) -- (1+\x/5,\y/5+0.8,2.6) -- (1.2+\x/5,\y/5+0.8,2.6) -- (1.2+\x/5,\y/5+0.6,2.6)  -- (1+\x/5,\y/5+0.6,2.6);
			}
		}
		\foreach \y in {1,2,3} {  
			\foreach \x in {1,2,3} {  
				\filldraw[fill=green] (1+\x/5,\y/5,2.6) -- (1+\x/5,\y/5+0.2,2.6) -- (1.2+\x/5,\y/5+0.2,2.6) -- (1.2+\x/5,\y/5,2.6)  -- (1+\x/5,\y/5,2.6);
			}
		}		
		\foreach \y in {1,2,3} {  
			\foreach \x in {1,2,3} {  
				\filldraw[fill=yellow] (0.4+\x/5,\y/5+1.2,2.6) -- (0.4+\x/5,\y/5+1.4,2.6) -- (0.6+\x/5,\y/5+1.4,2.6) -- (0.6+\x/5,\y/5+1.2,2.6)  -- (0.4+\x/5,\y/5+1.2,2.6);
			}
		}
		\foreach \y in {1,2,3} {  
			\foreach \x in {1,2,3} {  
				\filldraw[fill=red] (0.4+\x/5,\y/5+0.6,2.6) -- (0.4+\x/5,\y/5+0.8,2.6) -- (0.6+\x/5,\y/5+0.8,2.6) -- (0.6+\x/5,\y/5+0.6,2.6)  -- (0.4+\x/5,\y/5+0.6,2.6);
			}
		}
		\foreach \y in {1,2,3} {  
			\foreach \x in {1,2,3} {  
				\filldraw[fill=cyan] (0.4+\x/5,\y/5,2.6) -- (0.4+\x/5,\y/5+0.2,2.6) -- (0.6+\x/5,\y/5+0.2,2.6) -- (0.6+\x/5,\y/5,2.6)  -- (0.4+\x/5,\y/5,2.6);
			}
		}
		\foreach \y in {1,2,3} {  
			\foreach \x in {1,2,3} {  
				\filldraw[fill=green] (-0.2+\x/5,\y/5+1.2,2.6) -- (-0.2+\x/5,\y/5+1.4,2.6) -- (\x/5,\y/5+1.4,2.6) -- (\x/5,\y/5+1.2,2.6)  -- (-0.2+\x/5,\y/5+1.2,2.6);
			}
		}
		\foreach \y in {1,2,3} {  
			\foreach \x in {1,2,3} {  
				\filldraw[fill=blue] (-0.2+\x/5,\y/5+0.6,2.6) -- (-0.2+\x/5,\y/5+0.8,2.6) -- (\x/5,\y/5+0.8,2.6) -- (\x/5,\y/5+0.6,2.6)  -- (-0.2+\x/5,\y/5+0.6,2.6);
			}
		}
		\foreach \y in {1,2,3} {  
			\foreach \x in {1,2,3} {  
				\filldraw[fill=magenta] (-0.2+\x/5,\y/5,2.6) -- (-0.2+\x/5,\y/5+0.2,2.6) -- (\x/5,\y/5+0.2,2.6) -- (\x/5,\y/5,2.6)  -- (-0.2+\x/5,\y/5,2.6);
			}
		}
		
		\foreach \y in {1,2,3} {  
			\foreach \x in {1,2,3} {  
				\filldraw[fill=red] (1.8-\y/5,2,2.8-\x/5) -- (1.8-\y/5,2,2.6-\x/5) -- (2-\y/5,2,2.6-\x/5) -- (2-\y/5,2,2.8-\x/5) -- (2-\y/5,2,2.8-\x/5);
			}
		}		
		\foreach \y in {1,2,3} {  
			\foreach \x in {1,2,3} {  
				\filldraw[fill=yellow] (1.2-\y/5,2,2.8-\x/5) -- (1.2-\y/5,2,2.6-\x/5) -- (1.4-\y/5,2,2.6-\x/5) -- (1.4-\y/5,2,2.8-\x/5) -- (1.4-\y/5,2,2.8-\x/5);
			}
		}		
		\foreach \y in {1,2,3} {  
			\foreach \x in {1,2,3} {  
				\filldraw[fill=green] (0.6-\y/5,2,2.8-\x/5) -- (0.6-\y/5,2,2.6-\x/5) -- (0.8-\y/5,2,2.6-\x/5) -- (0.8-\y/5,2,2.8-\x/5) -- (0.8-\y/5,2,2.8-\x/5);
			}
		}
		\foreach \y in {1,2,3} {  
			\foreach \x in {1,2,3} {  
				\filldraw[fill=yellow] (1.8-\y/5,2,2.2-\x/5) -- (1.8-\y/5,2,2-\x/5) -- (2-\y/5,2,2-\x/5) -- (2-\y/5,2,2.2-\x/5) -- (2-\y/5,2,2.2-\x/5);
			}
		}		
		\foreach \y in {1,2,3} {  
			\foreach \x in {1,2,3} {  
				\filldraw[fill=cyan] (1.2-\y/5,2,2.2-\x/5) -- (1.2-\y/5,2,2-\x/5) -- (1.4-\y/5,2,2-\x/5) -- (1.4-\y/5,2,2.2-\x/5) -- (1.4-\y/5,2,2.2-\x/5);
			}
		}		
		\foreach \y in {1,2,3} {  
			\foreach \x in {1,2,3} {  
				\filldraw[fill=red] (0.6-\y/5,2,2.2-\x/5) -- (0.6-\y/5,2,2-\x/5) -- (0.8-\y/5,2,2-\x/5) -- (0.8-\y/5,2,2.2-\x/5) -- (0.8-\y/5,2,2.2-\x/5);
			}
		}		
		\foreach \y in {1,2,3} {  
			\foreach \x in {1,2,3} {  
				\filldraw[fill=magenta] (1.6-\y/5,2,1.6-\x/5) -- (1.6-\y/5,2,1.4-\x/5) -- (2-\y/5,2,1.4-\x/5) -- (2-\y/5,2,1.6-\x/5) -- (2-\y/5,2,1.6-\x/5);
			}
		}
		\foreach \y in {1,2,3} {  
			\foreach \x in {1,2,3} {  
				\filldraw[fill=blue] (1.2-\y/5,2,1.6-\x/5) -- (1.2-\y/5,2,1.4-\x/5) -- (1.4-\y/5,2,1.4-\x/5) -- (1.4-\y/5,2,1.6-\x/5) -- (1.4-\y/5,2,1.6-\x/5);
			}
		}		
		\foreach \y in {1,2,3} {  
			\foreach \x in {1,2,3} {  
				\filldraw[fill=green] (0.6-\y/5,2,1.6-\x/5) -- (0.6-\y/5,2,1.4-\x/5) -- (0.8-\y/5,2,1.4-\x/5) -- (0.8-\y/5,2,1.6-\x/5) -- (0.8-\y/5,2,1.6-\x/5);
			}
		}		
		\foreach \y in {1,2,3} {  
			\foreach \x in {1,2,3} {  
				\filldraw[fill=blue] (1.6-\y/5,2,1-\x/5) -- (1.6-\y/5,2,0.8-\x/5) -- (2-\y/5,2,0.8-\x/5) -- (2-\y/5,2,1-\x/5) -- (2-\y/5,2,1-\x/5);
			}
		}
		\foreach \y in {1,2,3} {  
			\foreach \x in {1,2,3} {  
				\filldraw[fill=yellow] (1.2-\y/5,2,1-\x/5) -- (1.2-\y/5,2,0.8-\x/5) -- (1.4-\y/5,2,0.8-\x/5) -- (1.4-\y/5,2,1-\x/5) -- (1.4-\y/5,2,1-\x/5);
			}
		}		
		\foreach \y in {1,2,3} {			
			\foreach \x in {1,2,3} {  
				\filldraw[fill=magenta] (0.6-\y/5,2,1-\x/5) -- (0.6-\y/5,2,0.8-\x/5) -- (0.8-\y/5,2,0.8-\x/5) -- (0.8-\y/5,2,1-\x/5) -- (0.8-\y/5,2,1-\x/5);
			}
		}

        \newcommand\Xoffs{0.6}
        \draw[decoration={brace,mirror,raise=-4pt},decorate] (-1,-1.1) -- node[below] {\textcolor{red}{$\Theta_1$},\textcolor{green}{$\Theta_2$},\textcolor{blue}{$\Theta_3$},\dots} (1.8,-1.1);
        \draw[decoration={brace,mirror},decorate] (-0.7,-1.7) -- (1.5,-1.7);        
        
		\draw[->] (-1+\Xoffs,-3) -- (-1+\Xoffs,-2);
		\draw[->] (-1+\Xoffs,-3) -- (0.6+\Xoffs,-3);
		\fill[yellow] (0+\Xoffs,-3) rectangle (0.2+\Xoffs,-2.3);
		\fill[red] (-1+\Xoffs,-3) rectangle (-0.8+\Xoffs,-2.5);
		\fill[green] (-0.8+\Xoffs,-3) rectangle (-0.6+\Xoffs,-2.7);
		\fill[blue] (-0.6+\Xoffs,-3) rectangle (-0.4+\Xoffs,-2.4);
		\fill[magenta] (-0.4+\Xoffs,-3) rectangle (.-0.2+\Xoffs,-2.9);
		\fill[cyan] (-0.2+\Xoffs,-3) rectangle (-0+\Xoffs,-2.1);
		
        \newcommand\Xoffsq{4}
        \newcommand\Yoffsq{0}
		\foreach \x in {1,2,3,4,5,6,7,8,9} {
			\foreach \y in {1,2,3,4,5,6,7,8,9} {
				\filldraw[fill=red] (\Xoffsq+\x/5,\y/5+\Yoffsq,2.6) -- (\Xoffsq+\x/5,\y/5+\Yoffsq+0.2,2.6) -- (\Xoffsq+0.2+\x/5,\y/5+\Yoffsq+0.2,2.6) -- (\Xoffsq+0.2+\x/5,\y/5+\Yoffsq,2.6)  -- (\Xoffsq+\x/5,\y/5+\Yoffsq,2.6);
			}
		}

		\foreach \y in {1,2,3,4,5,6,7,8,9} {
			\filldraw[fill=red] (\Xoffsq+2,\y/5+\Yoffsq,2.6) -- (\Xoffsq+2,\y/5+\Yoffsq+0.2,2.6) -- (\Xoffsq+2,\y/5+\Yoffsq+0.2,2.4) -- (\Xoffsq+2,\y/5+\Yoffsq,2.4)  -- (\Xoffsq+2,\y/5+\Yoffsq,2.6);
			\filldraw[fill=red] (\Xoffsq+2.2-\y/5,\Yoffsq+2,2.6) -- (\Xoffsq+2-\y/5,\Yoffsq+2,2.6) -- (\Xoffsq+2-\y/5,\Yoffsq+2,2.4) -- (\Xoffsq+2.2-\y/5,\Yoffsq+2,2.4)  -- (\Xoffsq+2.2-\y/5,\Yoffsq+2,2.6);
		}
		\foreach \y in {1,2,3,4,5,6,7,8,9} {
			\filldraw[fill=green] (\Xoffsq+2,\y/5+\Yoffsq,2.4) -- (\Xoffsq+2,\y/5+\Yoffsq+0.2,2.4) -- (\Xoffsq+2,\y/5+\Yoffsq+0.2,2.2) -- (\Xoffsq+2,\y/5+\Yoffsq,2.2)  -- (\Xoffsq+2,\y/5+\Yoffsq,2.4);
			\filldraw[fill=green] (\Xoffsq+2.2-\y/5,\Yoffsq+2,2.4) -- (\Xoffsq+2-\y/5,\Yoffsq+2,2.4) -- (\Xoffsq+2-\y/5,\Yoffsq+2,2.2) -- (\Xoffsq+2.2-\y/5,\Yoffsq+2,2.2)  -- (\Xoffsq+2.2-\y/5,\Yoffsq+2,2.4);
		}
		\foreach \y in {1,2,3,4,5,6,7,8,9} {
			\filldraw[fill=blue] (\Xoffsq+2,\y/5+\Yoffsq,2.2) -- (\Xoffsq+2,\y/5+\Yoffsq+0.2,2.2) -- (\Xoffsq+2,\y/5+\Yoffsq+0.2,2) -- (\Xoffsq+2,\y/5+\Yoffsq,2)  -- (\Xoffsq+2,\y/5+\Yoffsq,2.2);
			\filldraw[fill=blue] (\Xoffsq+2.2-\y/5,\Yoffsq+2,2.2) -- (\Xoffsq+2-\y/5,\Yoffsq+2,2.2) -- (\Xoffsq+2-\y/5,\Yoffsq+2,2) -- (\Xoffsq+2.2-\y/5,\Yoffsq+2,2)  -- (\Xoffsq+2.2-\y/5,\Yoffsq+2,2.2);
		}
		\foreach \y in {1,2,3,4,5,6,7,8,9} {
			\filldraw[fill=magenta] (\Xoffsq+2,\y/5+\Yoffsq,2) -- (\Xoffsq+2,\y/5+\Yoffsq+0.2,2) -- (\Xoffsq+2,\y/5+\Yoffsq+0.2,1.8) -- (\Xoffsq+2,\y/5+\Yoffsq,1.8)  -- (\Xoffsq+2,\y/5+\Yoffsq,2);
			\filldraw[fill=magenta] (\Xoffsq+2.2-\y/5,\Yoffsq+2,2) -- (\Xoffsq+2-\y/5,\Yoffsq+2,2) -- (\Xoffsq+2-\y/5,\Yoffsq+2,1.8) -- (\Xoffsq+2.2-\y/5,\Yoffsq+2,1.8)  -- (\Xoffsq+2.2-\y/5,\Yoffsq+2,2);
		}
		\foreach \y in {1,2,3,4,5,6,7,8,9} {
			\filldraw[fill=yellow] (\Xoffsq+2,\y/5+\Yoffsq,1.8) -- (\Xoffsq+2,\y/5+\Yoffsq+0.2,1.8) -- (\Xoffsq+2,\y/5+\Yoffsq+0.2,1.6) -- (\Xoffsq+2,\y/5+\Yoffsq,1.6)  -- (\Xoffsq+2,\y/5+\Yoffsq,1.8);
			\filldraw[fill=yellow] (\Xoffsq+2.2-\y/5,\Yoffsq+2,1.8) -- (\Xoffsq+2-\y/5,\Yoffsq+2,1.8) -- (\Xoffsq+2-\y/5,\Yoffsq+2,1.6) -- (\Xoffsq+2.2-\y/5,\Yoffsq+2,1.6)  -- (\Xoffsq+2.2-\y/5,\Yoffsq+2,1.8);
		}
		\foreach \y in {1,2,3,4,5,6,7,8,9} {
			\filldraw[fill=cyan] (\Xoffsq+2,\y/5+\Yoffsq,1.6) -- (\Xoffsq+2,\y/5+\Yoffsq+0.2,1.6) -- (\Xoffsq+2,\y/5+\Yoffsq+0.2,1.4) -- (\Xoffsq+2,\y/5+\Yoffsq,1.4)  -- (\Xoffsq+2,\y/5+\Yoffsq,1.6);
			\filldraw[fill=cyan] (\Xoffsq+2.2-\y/5,\Yoffsq+2,1.6) -- (\Xoffsq+2-\y/5,\Yoffsq+2,1.6) -- (\Xoffsq+2-\y/5,\Yoffsq+2,1.4) -- (\Xoffsq+2.2-\y/5,\Yoffsq+2,1.4)  -- (\Xoffsq+2.2-\y/5,\Yoffsq+2,1.6);
		}
		
		\foreach \y in {1,2,3,4,5,6,7,8,9} {
			\filldraw[fill=red] (\Xoffsq+2,\y/5+\Yoffsq,2.6-1.2) -- (\Xoffsq+2,\y/5+\Yoffsq+0.2,2.6-1.2) -- (\Xoffsq+2,\y/5+\Yoffsq+0.2,2.4-1.2) -- (\Xoffsq+2,\y/5+\Yoffsq,2.4-1.2)  -- (\Xoffsq+2,\y/5+\Yoffsq,2.6-1.2);
			\filldraw[fill=red] (\Xoffsq+2.2-\y/5,\Yoffsq+2,2.6-1.2) -- (\Xoffsq+2-\y/5,\Yoffsq+2,2.6-1.2) -- (\Xoffsq+2-\y/5,\Yoffsq+2,2.4-1.2) -- (\Xoffsq+2.2-\y/5,\Yoffsq+2,2.4-1.2)  -- (\Xoffsq+2.2-\y/5,\Yoffsq+2,2.6-1.2);
		}
		\foreach \y in {1,2,3,4,5,6,7,8,9} {
			\filldraw[fill=green] (\Xoffsq+2,\y/5+\Yoffsq,2.4-1.2) -- (\Xoffsq+2,\y/5+\Yoffsq+0.2,2.4-1.2) -- (\Xoffsq+2,\y/5+\Yoffsq+0.2,2.2-1.2) -- (\Xoffsq+2,\y/5+\Yoffsq,2.2-1.2)  -- (\Xoffsq+2,\y/5+\Yoffsq,2.4-1.2);
			\filldraw[fill=green] (\Xoffsq+2.2-\y/5,\Yoffsq+2,2.4-1.2) -- (\Xoffsq+2-\y/5,\Yoffsq+2,2.4-1.2) -- (\Xoffsq+2-\y/5,\Yoffsq+2,2.2-1.2) -- (\Xoffsq+2.2-\y/5,\Yoffsq+2,2.2-1.2)  -- (\Xoffsq+2.2-\y/5,\Yoffsq+2,2.4-1.2);
		}
		\foreach \y in {1,2,3,4,5,6,7,8,9} {
			\filldraw[fill=blue] (\Xoffsq+2,\y/5+\Yoffsq,2.2-1.2) -- (\Xoffsq+2,\y/5+\Yoffsq+0.2,2.2-1.2) -- (\Xoffsq+2,\y/5+\Yoffsq+0.2,2-1.2) -- (\Xoffsq+2,\y/5+\Yoffsq,2-1.2)  -- (\Xoffsq+2,\y/5+\Yoffsq,2.2-1.2);
			\filldraw[fill=blue] (\Xoffsq+2.2-\y/5,\Yoffsq+2,2.2-1.2) -- (\Xoffsq+2-\y/5,\Yoffsq+2,2.2-1.2) -- (\Xoffsq+2-\y/5,\Yoffsq+2,2-1.2) -- (\Xoffsq+2.2-\y/5,\Yoffsq+2,2-1.2)  -- (\Xoffsq+2.2-\y/5,\Yoffsq+2,2.2-1.2);
		}
		\foreach \y in {1,2,3,4,5,6,7,8,9} {
			\filldraw[fill=magenta] (\Xoffsq+2,\y/5+\Yoffsq,2-1.2) -- (\Xoffsq+2,\y/5+\Yoffsq+0.2,2-1.2) -- (\Xoffsq+2,\y/5+\Yoffsq+0.2,1.8-1.2) -- (\Xoffsq+2,\y/5+\Yoffsq,1.8-1.2)  -- (\Xoffsq+2,\y/5+\Yoffsq,2-1.2);
			\filldraw[fill=magenta] (\Xoffsq+2.2-\y/5,\Yoffsq+2,2-1.2) -- (\Xoffsq+2-\y/5,\Yoffsq+2,2-1.2) -- (\Xoffsq+2-\y/5,\Yoffsq+2,1.8-1.2) -- (\Xoffsq+2.2-\y/5,\Yoffsq+2,1.8-1.2)  -- (\Xoffsq+2.2-\y/5,\Yoffsq+2,2-1.2);
		}
		\foreach \y in {1,2,3,4,5,6,7,8,9} {
			\filldraw[fill=yellow] (\Xoffsq+2,\y/5+\Yoffsq,1.8-1.2) -- (\Xoffsq+2,\y/5+\Yoffsq+0.2,1.8-1.2) -- (\Xoffsq+2,\y/5+\Yoffsq+0.2,1.6-1.2) -- (\Xoffsq+2,\y/5+\Yoffsq,1.6-1.2)  -- (\Xoffsq+2,\y/5+\Yoffsq,1.8-1.2);
			\filldraw[fill=yellow] (\Xoffsq+2.2-\y/5,\Yoffsq+2,1.8-1.2) -- (\Xoffsq+2-\y/5,\Yoffsq+2,1.8-1.2) -- (\Xoffsq+2-\y/5,\Yoffsq+2,1.6-1.2) -- (\Xoffsq+2.2-\y/5,\Yoffsq+2,1.6-1.2)  -- (\Xoffsq+2.2-\y/5,\Yoffsq+2,1.8-1.2);
		}
		\foreach \y in {1,2,3,4,5,6,7,8,9} {
			\filldraw[fill=cyan] (\Xoffsq+2,\y/5+\Yoffsq,1.6-1.2) -- (\Xoffsq+2,\y/5+\Yoffsq+0.2,1.6-1.2) -- (\Xoffsq+2,\y/5+\Yoffsq+0.2,1.4-1.2) -- (\Xoffsq+2,\y/5+\Yoffsq,1.4-1.2)  -- (\Xoffsq+2,\y/5+\Yoffsq,1.6-1.2);
			\filldraw[fill=cyan] (\Xoffsq+2.2-\y/5,\Yoffsq+2,1.6-1.2) -- (\Xoffsq+2-\y/5,\Yoffsq+2,1.6-1.2) -- (\Xoffsq+2-\y/5,\Yoffsq+2,1.4-1.2) -- (\Xoffsq+2.2-\y/5,\Yoffsq+2,1.4-1.2)  -- (\Xoffsq+2.2-\y/5,\Yoffsq+2,1.6-1.2);
		}

		\draw[decoration={brace,mirror,raise=-4pt},decorate] (\Xoffsq-0.8,-1.1) -- (\Xoffsq+2,-1.1);	

		\draw[->] (-1.5+\Xoffsq,-2.2) -- (-1.5+\Xoffsq,-1.4);
		\draw[->] (-1.5+\Xoffsq,-2.2) -- (-0.3+\Xoffsq,-2.2);
    	\fill[red] (-1.5+\Xoffsq,-2.2) rectangle (-1.35+\Xoffsq,-1.8);
		\fill[red] (-1.35+\Xoffsq,-2.2) rectangle (-1.2+\Xoffsq,-2);
		\fill[red] (-1.2+\Xoffsq,-2.2) rectangle (-1.05+\Xoffsq,-1.7);
		\fill[red] (-1.05+\Xoffsq,-2.2) rectangle (-0.9+\Xoffsq,-1.4);
		\fill[red] (-0.9+\Xoffsq,-2.2) rectangle (-0.75+\Xoffsq,-2.1);
		\fill[red] (-0.75+\Xoffsq,-2.2) rectangle (-0.6+\Xoffsq,-1.7);

		\draw[->] (-1.5+1.4+\Xoffsq,-2.2) -- (-1.5+1.4+\Xoffsq,-1.4);
		\draw[->] (-1.5+1.4+\Xoffsq,-2.2) -- (-0.3+1.4+\Xoffsq,-2.2);
		\fill[green] (-1.5+1.4+\Xoffsq,-2.2) rectangle (-1.35+1.4+\Xoffsq,-1.8);
		\fill[green] (-1.35+1.4+\Xoffsq,-2.2) rectangle (-1.2+1.4+\Xoffsq,-2);
		\fill[green] (-1.2+1.4+\Xoffsq,-2.2) rectangle (-1.05+1.4+\Xoffsq,-1.7);
		\fill[green] (-1.05+1.4+\Xoffsq,-2.2) rectangle (-0.9+1.4+\Xoffsq,-1.4);
		\fill[green] (-0.9+1.4+\Xoffsq,-2.2) rectangle (-0.75+1.4+\Xoffsq,-2.1);
		\fill[green] (-0.75+1.4+\Xoffsq,-2.2) rectangle (-0.6+1.4+\Xoffsq,-1.7);
				
		\draw[->] (-1.5+2.8+\Xoffsq,-2.2) -- (-1.5+2.8+\Xoffsq,-1.4);
		\draw[->] (-1.5+2.8+\Xoffsq,-2.2) -- (-0.3+2.8+\Xoffsq,-2.2);
		\fill[blue] (-1.5+2.8+\Xoffsq,-2.2) rectangle (-1.35+2.8+\Xoffsq,-1.8);
		\fill[blue] (-1.35+2.8+\Xoffsq,-2.2) rectangle (-1.2+2.8+\Xoffsq,-2);
		\fill[blue] (-1.2+2.8+\Xoffsq,-2.2) rectangle (-1.05+2.8+\Xoffsq,-1.7);
		\fill[blue] (-1.05+2.8+\Xoffsq,-2.2) rectangle (-0.9+2.8+\Xoffsq,-1.4);
		\fill[blue] (-0.9+2.8+\Xoffsq,-2.2) rectangle (-0.75+2.8+\Xoffsq,-2.1);
		\fill[blue] (-0.75+2.8+\Xoffsq,-2.2) rectangle (-0.6+2.8+\Xoffsq,-1.7);
		
		\node[] at (\Xoffsq+3,-1.8) {\dots};
		
		\draw[decoration={brace,mirror,raise=-4pt},decorate] (\Xoffsq-1.5,-2.6) -- node[below] {$\Theta$}  (\Xoffsq+2.7,-2.6);	
	\end{tikzpicture}
	
	\end{center}
	\caption{Bag of Systems vs. System of Bags: While a BoS describes the video as a global distribution of spatially and temporally localized systems, an SoB describes the video as a temporally localized but spatially global distribution of features that changes over time according to one single system.}
	\label{fig:bos_sob}
\end{figure}
We refer to such an approach as \emph{System of Bags} (SoB) in reference to the successful Bags of Systems. Unlike Bags of Systems, where dynamic systems on spatially and temporally local scale are computed and "bags" thereof on a global scale are created, we compute histograms on a temporally local but spatially global scale and model their evolution over time by means of a dynamic system. Fig. \ref{fig:bos_sob} visualizes this distinction. The left side of the picture shows the procedure of BoS and related generative models. A video is first divided into spatiotemporal cubes and for each cube, a \emph{word}, e.g. a stack of dynamic system parameters $\Theta_i,i\in\{1,2,3,\dots\}$, is computed. Afterwards, a descriptor distribution over the complete video is estimated as the final representation of the video. The right side of the figure shows the procedure of generating a SoB. Distributions of descriptors are computed on temporally local but a spatially global scale. The temporal progression of these distributions is successively modeled as a dynamic system $\Theta$. Typically, the terms \emph{bag} and \emph{word} imply the usage of a learned codebook. However, this technique can be applied independently of codebooks and thus will be referred to as SoBs even in cases where the histograms were not computed with respect to a codebook.
 
\subsection{Related Work and our Contribution}
The temporal evolution of histograms can not be well described by linear dynamic systems. Instead, we employ the concept of \emph{Kernel Linear Dynamic Systems} (KLDS) which model the observations in a kernel feature space. The parameters describing the KLDS are the representations of the visual processes employed in his work. Using these descriptors in the context of supervised learning requires the definition of a dissimilarity measure. The available literature on KLDSs offers a manageable number of dissimilarity measures that perform insufficiently on the SoB descriptors used in this work, as will be shown in the experimental Section \ref{sec:exp}, which may be due to the fact that they put too much emphasis on the dynamic part of the KLDS parameters, whereas in the setting discussed in this work, static information, i.e. the information not related to the temporal context has a considerable significance. The KLDS itself was introduced in \cite{chan2007classifying} and motivated by the recognition of dynamic textures. As a dissimilarity measure, a kernelized version of the widely adopted Martin distance \cite{de2002subspace} was applied.

Modeling visual processes as SoBs i.e. KLDSs of histograms was employed in \cite{chaudhry2009histograms} for the classification of human actions. The authors modeled videos of human actions as streams of histograms of optical flow (HOOF). As a similarity measure, a Binet-Cauchy Maximum Singular Value kernel was applied. The work was further enhanced in \cite{motiian2013pairwise}, where human interactions were targeted. More generally, bags and histograms as representations of samples of multidimensional time signals were used both for human action \cite{ofli2014sequence} and dynamic scene \cite{feichtenhofer2014bags} recognition, but the temporal order was neglected in both cases. 

The novelty of this work is to explore the framework of SoB in the context of dynamic scene and large-scale dynamic texture recognition and to develop an appropriate dissimilarity measure. To this end, we adopt the framework of the \emph{alignment distance} from \cite{afsari2012group} and develop a kernelized version suitable for the comparison of SoBs. Unlike the mentioned dissimilarity measures for SoBs, impact of the static and dynamic components can be chosen depending on the employed generative image model. Besides, its property of being the square of a metric and its simple definition based on the Frobenius distance allows for the definition of \emph{Fr\'echet means} \cite{frechet1948les}. This is crucial for the classification via the \emph{nearest class center} (NCC) \cite{dubois2015characterization}. A part of this work is dedicated to the computation of abstract means of sets of KLDSs for avoiding the memory burden of \emph{nearest neighbor} (1-NN) classification.

The computation of alignment distances involves modeling the appearance of a visual process as an equivalence class of points on a Stiefel manifold. This is closely related to Grassmann based models. The authors of \cite{harandi2013dictionary} model visual processes as points on kernelized Grassmann manifolds, while our approach employs kernelized Stiefel manifolds in a similar manner. In particular, the authors propose to model the spaces of video frames, or of temporally localized features extracted from them, as sparse codes via points on Grassmann manifolds. Furthermore, a kernel dictionary learning approach for dynamic texture recognition was employed in \cite{quan2016equiangular}.

\subsection{Notation}
Boldfaced uppercase letters, e.g. $\mathbf{A},\mathbf{C}$ denote matrices and boldfaced lowercase letters e.g $\mathbf{v},\mathbf{w}$ denote vectors. Bold italic letters like $\boldsymbol{a}$ or $\boldsymbol{\alpha}$ refer to any members of metric spaces. The identity matrix in $\mathbb{R}^{n\times n}$ is written as $\mathbf{I}_n$ and a vector of ones as $\mathbf{1}_n$. For submatrices, the colon notation of Matlab is adopted, e.g. $\mathbf{V}_{1:n,1:n}$ for the left-upper $n\times n$ square submatrix of $\mathbf{V}$. Singular value and Eigenvalue decompositions are assumed to be sorted in a descending manner.
\section{Systems of Bags}
\subsection{Visual Processes as Streams of Histograms}
Given an ordered set of vectorized samples $\{\mathbf{s}_1\cdots\mathbf{s}_N\}$ of a multidimensional signal in time, let us assume that temporally local dynamics are negligible for the assignment of a class. This assumption can not be kept up when it comes to determining the sense of rotation of a wheel or a windmill from video footage but can be usually assumed to be valid for telling one scene video apart from another. In such cases, it is sensible to convert the signal to a matrix $\begin{bmatrix}\mathbf{y}_1\cdots\mathbf{y}_N\end{bmatrix}\in\mathbb{R}^{p\times N}$ of temporally localized feature vectors that capture the distinguishable characteristics on a spatially global but temporally local scale. 

Classification requires the generalization from one set of signals of one class to other signals of the same class. Since the entities in our model are temporally ordered sets of features, our aim is to learn a generative model that describes how the feature vectors develop over time. For many image classification scenarios, histogram-based feature vectors have proven successful. We do not pose any constraints on the histograms except that their entries are nonnegative and their $\ell_1$-norm is 1. Let $\mathbf{Y}=\begin{bmatrix}\mathbf{y}_1\cdots\mathbf{y}_N\end{bmatrix}\in\mathbb{R}^{p\times N}$ be a sample matrix of histogram vectors that were observed from a visual process over time. A common temporal model - one that is particularly popular in the modeling of dynamic textures \cite{doretto2003dynamic} - is a linear dynamic system (LDS) typically modeled as
\begin{equation}
\label{eq:LDS}
\begin{split}
\mathbf{x}_{t+1}&=\mathbf{A}\mathbf{x}_t+\mathbf{w}_t,\\
\mathbf{y}_{t}&=\boldsymbol{\mu}+\mathbf{C}\mathbf{x}_t+\mathbf{v}_t,
\end{split}
\end{equation}
where $\boldsymbol{\mu}\in\mathbb{R}^p$ is the expected value of the observations $\{\mathbf{y}_1,\mathbf{y}_2,\dots\}$ and $\mathbf{C}\in\mathbb{R}^{p\times n}$ is the \emph{observation matrix} which, together with $\boldsymbol{\mu}$ maps the internal state space vector $\mathbf{x}_t\in\mathbb{R}^n$ to its respective observation $\mathbf{y}\in\mathbb{R}^p$ at time $t$, unperturbed by noise. The observation noise vector $\mathbf{v}_t$ describes the model error. In the state space, the \emph{state transition matrix} $\mathbf{A}\in\mathbb{R}^{n\times n}$ models the expected temporal evolution of the state vector and the term $\mathbf{w}_t$ accounts for the process noise. The noise terms are assumed to be i.i.d. Gaussian. The parameters $\boldsymbol{\mu}$, $\mathbf{A}$ and $\mathbf{C}$ describe the predictable part of the dynamics of the system and are thus a natural choice for a feature representation of it.

\subsection{Kernelized Linear Dynamic Systems}
Sets of histograms can not be well modeled by linear vector spaces due to the intrinsic structure of histogram manifolds \cite{chaudhry2009histograms} and thus a non-linear model is preferred. Let the function $\varphi:\mathbb{R}^p\rightarrow\mathcal{H}$ be a feature space mapping of histograms to a feature space corresponding to an appropriate histogram kernel
\begin{equation}
\kappa:\mathbb{R}^{p\times M}\times\mathbb{R}^{p\times N}\rightarrow\mathbb{R}^{M\times N}.
\end{equation}
In other words, for two histograms $\mathbf{y}_1,\mathbf{y}_2\in\mathbb{R}^p$, the inner product fo their feature space mapping in the Hilbert space $\mathcal{H}$ can be written as
\begin{equation}
\langle \varphi(\mathbf{y}_1),\varphi(\mathbf{y}_2)\rangle_\mathcal{H}=\kappa(\mathbf{y}_1,\mathbf{y}_2).
\end{equation}
We assume that $\kappa$ operates on matrices and returns a matrix of kernel values for each pair of columns of the input matrices. A number of kernels are available for probabilistic models and, in particular, histograms. Among the most popular are the Bhattacharrya kernel, the histogram intersection kernel \cite{barla2003histogram} and the $\chi^2$-kernel. We will restrict ourselves to the $\chi^2$-kernel which is defined as 
\begin{equation}
	\kappa_{\chi^2}(\mathbf{y}_1,\mathbf{y}_2)=\exp\left(-\frac{1}{2}\sum_{i\in\mathcal{S}_1\cup\mathcal{S}_2}\frac{(y_{1,i}-y_{2,i})^2}{y_{1,i}+y_{2,i}}\right)
	\label{eq:chi}
	\end{equation}
	for a pair of histograms $\mathbf{y}_1,\mathbf{y}_2$, where $\mathcal{S}_1,\mathcal{S}_2$ denote the supports of $\mathbf{y}_1$ and $\mathbf{y}_2$, respectively.

Since kernel feature spaces are separable \cite{scholkopf2002learning}, we can think of  linear operators that map from real-valued euclidean vectors to $\mathcal{H}$ as matrices, where by \emph{matrix} a touple of elements in $\mathcal{H}$ is meant. Specifically, a matrix $\mathbf{F}=\begin{bmatrix}
\boldsymbol{f}_1\cdots\boldsymbol{f}_M
\end{bmatrix}\in\mathcal{H}^M$ 
represents the operator
\begin{equation}
\begin{split}
\boldsymbol{F}:\mathbb{R}^{M\times N}&\rightarrow\mathcal{H}^N,\\
\mathbf{D}&\mapsto 
\begin{bmatrix}
\sum_{i=1}^{M}d_{i,1}\boldsymbol{f}_i \cdots \sum_{i=1}^{M}d_{i,N}\boldsymbol{f}_i 
\end{bmatrix}.
\end{split}
\end{equation}
Beyond that, for two matrices $\mathbf{F}\in\mathcal{H}^m,\mathbf{G}\in\mathcal{H}^n$, we define the product
\begin{equation}
\mathbf{F}^\top\mathbf{G}=\begin{bmatrix}
\langle \boldsymbol{f}_1,\boldsymbol{g}_1\rangle_\mathcal{H}&\cdots&\langle\boldsymbol{f}_1,\boldsymbol{g}_n\rangle_\mathcal{H}\\
\vdots & \ddots & \vdots\\
\langle\boldsymbol{f}_m,\boldsymbol{g}_1\rangle_\mathcal{H}&\cdots&\langle\boldsymbol{f}_m,\boldsymbol{g}_n\rangle_\mathcal{H}\\
\end{bmatrix}.
\end{equation}
From this follow the definitions of the respective Frobenius inner product $\mathrm{tr}(\mathbf{F}^\top\mathbf{G})$ for $M=N$ and the Frobenius norm
\begin{equation}
\|\mathbf{F}\|_F=\sqrt{\mathrm{tr}(\mathbf{F}^\top\mathbf{F})}.
\end{equation}

A KLDS is defined as
\begin{equation}
	\label{eq:KLDS}
	\begin{split}
		\mathbf{x}_{t+1}&=\mathbf{A}\mathbf{x}_t+\mathbf{w}_t,\\
		\varphi(\mathbf{y}_{t})&=\boldsymbol{\mu}+\mathbf{C}\mathbf{x}_t+\mathbf{v}_t,
	\end{split}
\end{equation}
Unlike equation \eqref{eq:LDS}, the observation $\mathbf{y}_t$ is not modeled directly, but in terms of its feature space mapping $\varphi(\mathbf{y}_t)\in\mathcal{H}$.  The matrix $\mathbf{C}$, denoted \emph{feature space observer} in the following, is typically not directly available, since $\varphi$ is usually not given and described implicitly via the kernel trick. The same holds for the \emph{feature space bias} $\boldsymbol{\mu}$. However, given a set of observations $\{\mathbf{y}_1,\dots\textbf{y}_N\}$ such that
\begin{equation}
	\{\mathbf{C}_{:,1},\dots,\mathbf{C}_{:,n}\}\cup\{\boldsymbol{\mu}\}\subset\mathrm{span}\left(\{\varphi(\mathbf{y}_1),\dots,\varphi(\textbf{y}_N)\}\right)
\end{equation}
holds, let us define
\begin{equation}
\mathbf{\Phi}=\begin{bmatrix}
\varphi(\mathbf{y}_1)\cdots\varphi(\textbf{y}_N)
\end{bmatrix}\in\mathcal{H}^N.
\end{equation}
Then there exists a coefficient matrix $\boldsymbol{\alpha}\in\mathbb{R}^{N\times n}$ and a coefficient vector $\boldsymbol{\beta}\in\mathbb{R}^{N}$ for which the following relations are valid.
\begin{equation}
	\mathbf{C}=\mathbf{\Phi}\boldsymbol{\alpha},\ \boldsymbol{\mu}=\mathbf{\Phi}\boldsymbol{\beta}.
\end{equation}
Thus, the KLDS \eqref{eq:KLDS} can be equivalently described via $\mathbf{A}$ and $\boldsymbol{\alpha}$, along with $\boldsymbol{\beta}$ and a sample matrix $\mathbf{Y}=\begin{bmatrix}\mathbf{y}_1\cdots\mathbf{y}_N\end{bmatrix}$ as follows.
\begin{equation}
	\label{eq:KLDS_rewritten}
	\begin{split}
		\mathbf{x}_{t+1}&=\mathbf{A}\mathbf{x}_t+\mathbf{w}_t,\\
		\varphi(\mathbf{y}_{t})&=\mathbf{\Phi}(\boldsymbol{\beta}+\boldsymbol{\alpha}\mathbf{x}_t)+\mathbf{v}_t.
	\end{split}
\end{equation}
Given two feature space observers $\mathbf{C}_1,\mathbf{C}_2$ and feature space biases $\boldsymbol{\mu}_1,\boldsymbol{\mu}_2$, described by the sample matrices $\mathbf{Y}_1,\mathbf{Y}_2$ and the coefficient parameters $\boldsymbol{\alpha}_1,\boldsymbol{\alpha}_2,\boldsymbol{\beta}_1,\boldsymbol{\beta}_2$, the relations
\begin{equation}
\begin{split}
\mathbf{C}_i^\top\mathbf{C}_j&=\boldsymbol{\alpha}_i^\top\kappa(\mathbf{Y}_i,\mathbf{Y}_j)\boldsymbol{\alpha}_j,\\
\boldsymbol{\mu}_i^\top\boldsymbol{\mu}_j&=\boldsymbol{\beta}_i^\top\kappa(\mathbf{Y}_i,\mathbf{Y}_j)\boldsymbol{\beta}_j,\ i,j\in\{1,2\}
\end{split}
\end{equation}
can be concluded. For the kernel \eqref{eq:chi}, the set of feature representations of histograms is bounded. This follows directly from the fact that for any histogram, the canonical norm induced by \eqref{eq:chi} is $1$.  We can thus assume, that the underlying system is stable. This is imposed by constraining the spectral norm of the state transition matrix to
\begin{equation}
	\|\mathbf{A}\|_2<1.
\end{equation}

An algorithm based on Kernel PCA to extract $\mathbf{A}$ and $\boldsymbol{\alpha}$ from a set of observations of a system was provided in \cite{chan2007classifying}. Algorithm \ref{alg:llama} incorporates the procedure. The state space dimension $n$ has to be fixed manually. It should be large enough to capture the variation of the data which can be measured by observing the magnitudes of the eigenvalues of the Gram matrix $\mathbf{K}$, and small enough to make the computation of the state transition matrix $\mathbf{A}$ feasible.
\begin{algorithm}
	\KwIn{Data matrix $\mathbf{Y}\in\mathbb R^{p\times N}$, state space dimension $n\in\mathbb{N}$}
	\KwOut{KLDS parameters $\mathbf{A}\in\mathbb{R}^{n\times n},\mathbf{Y}\in\mathbb R^{p\times N},\boldsymbol{\alpha}\in\mathbb{R}^{N\times n},\boldsymbol{\beta}\in\mathbb{R}^N$}
	\BlankLine
	$\mathbf{K}\leftarrow (\mathbf{I-\frac{1}{N}\mathbf{1}_N\mathbf{1}_N^\top})\kappa(\mathbf{Y},\mathbf{Y})(\mathbf{I-\frac{1}{N}\mathbf{1}_N\mathbf{1}_N^\top})$;\\
	$(\mathbf{V},\mathbf{\Lambda})\leftarrow\mathrm{EVD}(\mathbf{K})$;\\
	$\boldsymbol{\alpha}\leftarrow\mathbf{V}_{:,1:n}\Lambda_{1:n,1:n}^{-\frac{1}{2}}$;\\
	$\mathbf{X}\leftarrow \Lambda_{1:n,1:n}^\frac{1}{2}\mathbf{V}_{1:n}^\top$;\\
	$\mathbf{A}\leftarrow \argmin_{\tilde{\mathbf{A}}}\|\mathbf{X}_{:,2:N}-\mathbf{A}\mathbf{X}_{1:N-1}\|$ s.t. $\|\mathbf{A}\|_2<1$;\\
	$\boldsymbol{\beta}\leftarrow\frac{1}{N}\mathbf{1}_N$;
	\BlankLine
	\caption{Extraction of KLDS parameters}
	\label{alg:llama}
\end{algorithm}

\section{Kernelized Alignment Distances}
\subsection{State Space Bases and Invariance}
Given two KLDSs described by  $\Theta_1=\left(\mathbf{A}_1,\mathbf{Y}_1,\boldsymbol{\alpha}_1,\boldsymbol{\beta}_1\right)$ and ${\Theta}_2=\left(\mathbf{A}_2,\mathbf{Y}_2,\boldsymbol{\alpha}_2,\boldsymbol{\beta}_2\right)$, respectively, one of the most elementary machine learning tasks is comparing the two by means of a dissimilarity measure. A natural choice is a linear combination of the squared Frobenius distances of the dynamic parameters:
\begin{equation}
	\label{eq:frobDist}
	\begin{split}
d_F(\Theta_1,\Theta_2))^2
		=&\lambda_A\|\mathbf{A}_1-\mathbf{A}_2\|_F^2+\|\mathbf{C}_1-\mathbf{C}_2\|_F^2\\
		&+\lambda_{\mu}\|\boldsymbol{\mu}_1-\boldsymbol{\mu}_2\|^2
		\end{split}
\end{equation}
Note that \eqref{eq:frobDist} differs from the formulation in \cite{afsari2012group} by the term $\|\boldsymbol{\mu}_1-\boldsymbol{\mu}_2\|^2$ and the missing of the process noise covariance. Since we assume that the temporally local histograms are already quite discriminative for each visual process, it is sensible to consider the feature space bias in the distance measure. By contrast, the process noise is of little importance and will be neglected. The parameters $\lambda_A$ and $\lambda_\mu$ are real and positive. They incorporate the discriminatory power of each aspect of deterministic part of the dynamics. Their choice is always a matter of consideration and depends on the specific problem. When a sufficient number of training samples are available, cross validation can be used for finding the best values. Some general guidelines can be inferred from the roles the KLDS parameters play in the motion equation \eqref{eq:KLDS}. In particular, large values for $\lambda_A$ should be used for scenarios, where the appearance has little discriminatory power, e.g. when similar objects with diverse movements are to be distinguished. A value well above $1$ for $\lambda_\mu$ should be employed, when isolated frames of the videos provide much discriminatory power and a value close to $0$ for the opposite case.

To facilitate computations, we rewrite \eqref{eq:frobDist} in terms of the trace product and split it up as
\begin{equation}
d_F(\Theta_1,\Theta_2)^2=\tau(\Theta_1,\Theta_2)-2\rho(\Theta_1,\Theta_2),
\end{equation}
with
\begin{equation}
\begin{split}
\tau(\Theta_1,\Theta_2)=&\lambda_A(\mathrm{tr}(\mathbf{A}_1^\top\mathbf{A}_1)+\mathrm{tr}(\mathbf{A}_2^\top\mathbf{A}_2))\\
&+\mathrm{tr}(\boldsymbol{\alpha}_1^\top\kappa(\mathbf{Y}_1,\mathbf{Y}_1)\boldsymbol{\alpha}_1)\\
&+\mathrm{tr}(\boldsymbol{\alpha}_2^\top\kappa(\mathbf{Y}_2,\mathbf{Y}_2)\boldsymbol{\alpha}_2)+\lambda_\mu\|\boldsymbol{\mu}_1-\boldsymbol{\mu_2}\|^2
\end{split}
\label{eq:tau}
\end{equation}
and
\begin{equation}
\rho(\Theta_1,\Theta_2)=\lambda_A\mathrm{tr}(\mathbf{A}_1^\top\mathbf{A}_2)+\mathrm{tr}(\boldsymbol{\alpha}_1^\top\kappa(\mathbf{Y}_1,\mathbf{Y}_2)\boldsymbol{\alpha}_2).
\end{equation}
The advantage of \eqref{eq:frobDist} is that the Frobenius distance is well studied and easy to interpret. However, this choice has the drawback that it is ambiguous. To see this, let $\mathbf{P}\in\mathbb{R}^{n\times n}$ be invertible. For an arbitrary parameter touple $\left(\mathbf{A},\mathbf{Y},\boldsymbol{\alpha},\boldsymbol{\beta}\right)$, consider the transformation
\begin{equation}
	\mathbf{P}\cdot\left(\mathbf{A},\mathbf{Y},\boldsymbol{\alpha},\boldsymbol{\beta}\right)=\left(\mathbf{P}^{-1}\mathbf{A}\mathbf{P},\mathbf{Y},\boldsymbol{\alpha}\mathbf{P},\boldsymbol{\beta}\right).
	\label{eq:P_tran}
\end{equation}
Substituting it into \eqref{eq:KLDS_rewritten} and neglecting the noise terms indicates that it describes the same dynamics as $(\mathbf{A},\mathbf{Y},\boldsymbol{\alpha},\boldsymbol{\beta})$, but $d_F\left((\mathbf{A},\mathbf{Y},\boldsymbol{\alpha},\boldsymbol{\beta}),\mathbf{P}\cdot (\mathbf{A},\mathbf{Y},\boldsymbol{\alpha},\boldsymbol{\beta})\right)^2$ does not vanish in general. This is undesirable, since a distance measure should account for ambiguities of particular representations. It is possible to partially accommodate these ambiguities by imposing the constraint that the columns of the feature space  observer $\mathbf{C}$ must be orthogonal, i.e
\begin{equation}
	\mathbf{C}^\top\mathbf{C}= \boldsymbol{\alpha}^\top\kappa(\mathbf{Y},\mathbf{Y})\boldsymbol{\alpha}=\mathbf{I}_n.
\end{equation}
For any representation $\left(\mathbf{A},\mathbf{Y},\boldsymbol{\alpha},\boldsymbol{\beta}\right)$, a change of state space basis $\mathbf{P}$ can be found such that the transformation \eqref{eq:P_tran} satisfies this constraint. Furthermore, this constraint is fulfilled for any representation extracted with Algorithm \ref{alg:llama}. We formalize it by defining the set of valid, stable KLDSs as
\begin{equation}
	\begin{split}
		\mathcal{K}_{n,p,\kappa}=&\{(\mathbf{A},\mathbf{Y},\boldsymbol{\alpha},\boldsymbol{\beta})\\
		&\in\mathbb{R}^{n\times n}\times\mathbb{R}^{p\times N}\times\mathbb{R}^{N\times n}\times\mathbb{R}^N| N\in\mathbb{N},\\
		&\|\mathbf{A}\|_2<1,\boldsymbol{\alpha}^\top\kappa(\mathbf{Y},\mathbf{Y})\boldsymbol{\alpha}=\mathbf{I}_n\}.
	\end{split}
\end{equation}
For any $\Theta\in\mathcal{K}_{n,p,\kappa}$, a transformation $\mathbf{Q}\cdot\Theta$ is a member of $\mathcal{K}_{n,p,\kappa}$ if and only if $\mathbf{Q}$ is orthogonal. Furthermore, for any two descriptors 
$\Theta_1=(\mathbf{A}_1,\mathbf{Y}_1,\boldsymbol{\alpha}_1,\boldsymbol{\beta}_1)\in \mathcal{K}_{n,p,\kappa}$ and $\Theta_2=(\mathbf{A}_2,\mathbf{Y}_2,\boldsymbol{\alpha}_2,\boldsymbol{\beta}_2)\in \mathcal{K}_{n,p,\kappa}$, it can be shown that
\begin{equation}
d_F(\Theta_1,\mathbf{Q}\cdot\Theta_2)^2=d_F(\mathbf{Q}^\top\cdot\Theta_1,\Theta_2)^2\;\ \forall\mathbf{Q}\in O(n)
\label{eq:sym}
\end{equation}
holds and equation \eqref{eq:tau} simplifies to
\begin{equation}
	\begin{split}
		\tau(\Theta_1,\Theta_2)=&\lambda_A(\mathrm{tr}(\mathbf{A}_1^\top\mathbf{A}_1)+\mathrm{tr}(\mathbf{A}_2^\top\mathbf{A}_2))\\
		&+\lambda_\mu\|\boldsymbol{\mu}_1-\boldsymbol{\mu_2}\|^2+2n.
	\end{split}
	\label{eq:tau_smpl}
\end{equation}
\subsection{Alignment Distance for KLDSs}
In the following we define a non-ambiguous dissimilarity measure on KLDSs.
The remaining ambiguity of $d_F^2$ with respect to orthogonal transformations suggests operating on equivalence classes of $\mathcal{K}_{n,p,\kappa}$, rather than itself. Let $\mathcal{OK}_{n,p,\kappa}$ denote the quotient of $\mathcal{K}_{n,p,\kappa}$ induced by the equivalence relation
\begin{equation}
	\{(\Theta_1,\Theta_2) |\ \exists \mathbf{Q}\in O(n)\ \mathrm{s.t.}\ \Theta_1=\mathbf{Q}\cdot\Theta_2\}.
	\label{eq:equi_class}
\end{equation}
A dissimilarity measure on $\mathcal{OK}_{n,p,\kappa}$ is
\begin{equation}
	\begin{split}
		d_\mathcal{O}(\Theta_1,\Theta_2)^2	=&\min_{\mathbf{Q}\in O(n)}d_F(\Theta_1,\mathbf{Q}\cdot\Theta_2)^2\\
		=&\tau(\Theta_1,\Theta_2)-2\max_{\mathbf{Q}\in O(n)}\rho(\Theta_1,\mathbf{Q}\cdot\Theta_2).
	\end{split}
	\label{eq:d_o}
\end{equation}

The square root $d_\mathcal{O}$  of \eqref{eq:d_o} is actually a metric on $\mathcal{OK}_{n,p,\kappa}$. To see the positive definiteness, we first observe that its square $d_\mathcal{O}^2$ is real and non-negative by definition. Neither can it be zero unless there is an orthogonal $\mathbf{Q}$, such that $\Theta_1=\mathbf{Q}\cdot\Theta_2$, since otherwise $2\rho$ would be smaller than $\tau$ in \eqref{eq:d_o}, due to the Cauchy-Schwarz inequality.  But this makes $\Theta_1$ and $\Theta_2$ member of the same equivalence class of \eqref{eq:equi_class}. The symmetry property follows directly form \eqref{eq:sym}:
\begin{equation}
\begin{split}
d_\mathcal{O}(\Theta_1,\Theta_2)&=
\min_{\mathbf{Q}\in O(n)}(\Theta_1,\mathbf{Q}\cdot\Theta_2)\\
&=\min_{\mathbf{Q}\in O(n)}d_F(\mathbf{Q}^\top\cdot\Theta_1,\Theta_2)\\
&=\min_{\mathbf{Q}\in O(n)}d_F(\Theta_2,\mathbf{Q}^\top\cdot\Theta_1)\\
&=d_\mathcal{O}(\Theta_2,\Theta_1).
\end{split}
\end{equation}
As for the triangle inequality, consider the three systems $\Theta_1, \Theta_2$, and $\Theta_3$, with $\mathbf{Q}_{12}$ and $\mathbf{Q}_{23}$ being the orthogonal matrices that solve \eqref{eq:d_o} for the respective pairs. Since $d_F$ is a metric on $\mathcal{K}_{n,p,\kappa}$, we can conclude the  relation
\begin{equation}
\begin{split}
&d_\mathcal{O}(\Theta_1,\Theta_2)+d_\mathcal{O}(\Theta_2,\Theta_3)\\
=&d_F(\mathbf{Q}_{12}^\top\cdot\Theta_1,\Theta_2)+d_F(\Theta_2,\mathbf{Q}_{23}\cdot\Theta_3)\\
\geq&d_F(\mathbf{Q}_{12}^\top\cdot\Theta_1,\mathbf{Q}_{23}\cdot\Theta_3)=d_F(\Theta_1,\mathbf{Q}_{12}\mathbf{Q}_{23}\cdot\Theta_3)\\
\geq&d_\mathcal{O}(\Theta_1,\Theta_3).
\end{split}
\end{equation}
We refer to $d_\mathcal{O}$ as the \emph{alignment metric}, as opposed to its square $d_\mathcal{O}^2$, \emph{alignment distance}. The reason for this distinction is that the metric is helpful for a mathematical interpretation, while for the definition and implementation of the algorithms only the distance is of interest.

\subsection{Jacobi-type Method for Computing the Alignment Distance}
This subsection aims at solving \eqref{eq:d_o}, which boils down to finding an orthogonal maximizer of 
\begin{equation}
\begin{split}
\rho(\Theta_1,&\mathbf{Q}\cdot\Theta_2)=\lambda_A\mathrm{tr}(\mathbf{A}_1^\top\mathbf{Q}^\top\mathbf{A}_2\mathbf{Q})\\
&+\mathrm{tr}(\boldsymbol{\alpha}_1^\top\kappa(\mathbf{Y}_1,\mathbf{Y}_2)\boldsymbol{\alpha}_2\mathbf{Q}).
\end{split}
\label{eq:rho_max}
\end{equation}
The set of $O(n)$ is not connected, but consists of the two connected components $SO(n)$ and $O(n)\setminus SO(n)$. For the sake of simplicity, we treat these two cases separately. The authors of \cite{jimenez2013fast} propose to compute the alignment distance of classical LDSs by writing matrices in $SO(n)$  as products of \emph{Givens} rotations, i.e. as
\begin{equation}
\prod_{p=1}^{n}\prod_{q=p+1}^{n}\mathbf{G}_{p,q}(c,s),
	\label{eq:qprod}
\end{equation}
in which $\mathbf{G}_{p,q}(c,s)$ describes a matrix that performs a rotation in the plane spanned by the coordinate axes $p$ and $q$:
\begin{equation}
	\mathbf{G}_{p,q}(c,s)=
	\small{\begin{bmatrix}   1   & \cdots &    0   & \cdots &    0   & \cdots &    0   \\
		\vdots & \ddots & \vdots &        & \vdots &        & \vdots \\
		0   & \cdots &    c   & \cdots &    s   & \cdots &    0   \\
		\vdots &        & \vdots & \ddots & \vdots &        & \vdots \\
		0   & \cdots &   -s   & \cdots &    c   & \cdots &    0   \\
		\vdots &        & \vdots &        & \vdots & \ddots & \vdots \\
		0   & \cdots &    0   & \cdots &    0   & \cdots &    1
		\end{bmatrix}}.
\end{equation}
The real numbers $s,c$ denote the sine and cosine of the rotation angle and appear in the rows and columns with the indexes $p$ and $q$, respectively. Consequentially, they conform 
\begin{eqnarray}
	s^2+c^2=1\ \mathrm{and}\ -1\leq s \leq 1.
	\label{eq:c_s_cond}
\end{eqnarray}
Now the optimization problem \eqref{eq:d_o} can be approached by maximizing $\rho$ for each $p$ and $q$ individually \cite{jimenez2013fast}:
\begin{algorithm}
	\KwIn{Pair of KLDS descriptors  $\Theta_1=(\mathbf{A}_1,\mathbf{Y}_1,\boldsymbol{\alpha_1},\boldsymbol{\beta}_1),\Theta_2=(\mathbf{A}_2,\mathbf{Y}_2,\boldsymbol{\alpha_2},\boldsymbol{\beta}_2)\in\mathcal{K}_{n,p,\kappa}$, initialization $\mathbf{Q}$, numerical tolerance $\epsilon\geq 0$}
	\KwOut{Minimizer $\hat{\mathbf{Q}}$ and result $d=d_\mathcal{O}(\Theta_1,\Theta_2)^2$ of problem \eqref{eq:d_o}}
	\BlankLine
	$\Theta\leftarrow\Theta_1$;\\
	$\tilde{\Theta}\leftarrow\mathbf{Q}\cdot\Theta_2$;\\
	\While{$\rho(\Theta_1,\mathbf{Q}\cdot\Theta_2)$ not converged}{
		\For{$p\in\{1,\dots,n-1\}$ and $q\in\{p+1,\dots,n\}$}{				  $\hat{s}\leftarrow\argmax_{\substack{s\in [-1,1],\\s^2+c^2=1}}\rho(\Theta,\mathbf{G}_{p,q}(c,s)\cdot\tilde{\Theta})$;\\\vspace*{1mm}
				$\hat{c}\leftarrow \mathrm{sgn}(\hat{c})\sqrt{1-\hat{s}^2}$;\\
				$\tilde{\Theta}\leftarrow \mathbf{G}_{p,q}(\hat{c},\hat{s})\cdot\tilde{\Theta}$;\\
				$\mathbf{Q}\leftarrow \mathbf{Q}\mathbf{G}_{p,q}(\hat{c},\hat{s})$\tcc*{Eq. \eqref{eq:qprod}}		
		}
	}
	$d\leftarrow d_F(\Theta,\tilde{\Theta})^2$;\\
	$\hat{\mathbf{Q}}\leftarrow \mathbf{Q}$;
	\BlankLine
	\caption{Computation of $d_O(\Theta_1,\Theta_2)^2$}
	\label{alg:alpaca}
\end{algorithm}
In each iteration, Algorithm \ref{alg:alpaca} \emph{sweeps} through all possible combinations of two-dimensional rotations and determines the sine minimizing the cost function \eqref{eq:d_o} for each one of them.
It repeats the procedure until a complete sweep does not significantly alter the cost function. The scalar optimization problem
\begin{equation}
\begin{split}
\hat{s}=\argmax_s \rho(\Theta,\mathbf{G}_{p,q}(c,s)\cdot\tilde{\Theta}),\\
\mathrm{s.t.\ } s\in [-1,1],\ s^2+c^2=1,
\end{split}
\end{equation}
can be solved analytically. With $\Theta=(\mathbf{A},\mathbf{Y},\boldsymbol{\alpha},\boldsymbol{\beta})$ and $\tilde{\Theta}=(\tilde{\mathbf{A}},\tilde{\mathbf{Y}},\tilde{\boldsymbol{\alpha}},\tilde{\boldsymbol{\beta}})$, it can be reformulated as
\begin{equation}
	\begin{split}
		\hat{s}=&\argmax_{s\in [-1,1],s^2+c^2=1}\lambda_A\mathrm{tr}(\mathbf{A}\mathbf{G}_{p,q}(c,s)^\top\tilde{\mathbf{A}}\mathbf{G}_{p,q}(c,s))\\
		&+\mathrm{tr}(\boldsymbol{\alpha}^\top\kappa(\mathbf{Y},\tilde{\mathbf{Y}})\tilde{\boldsymbol{\alpha}})\mathbf{G}_{p,q}(c,s)).
	\end{split}
	\label{eq:jacobi_eq}
\end{equation}
It can be observed that \eqref{eq:jacobi_eq} is quadratic in $c$ and $s$. Since constant offsets are irrelevant for maximization, we can write it as
\begin{equation}
\hat{s}=\argmax_{s\in [-1,1],\ c=\pm\sqrt{1-s^2}}k_0 c^2+k_1 s^2+k_2 cs+k_3 c+k_4 s,
\end{equation}
where the factors $k_0,\dots,k_4$ can be determined by writing the trace products as sums of matrix elements \cite{jimenez2013fast} and eliminating  the constant offset. Substituting the second constraint yields
\begin{equation}
\hat{s}=\argmax_{s\in [-1,1]}(k_1-k_0) s^2\pm(k_2 s +k_3) \sqrt{1-s^2}+k_4 s.
\end{equation}
If the optimum is not at the boundaries of $[-1,1]$, it is at a critical point, i.e. at a point with a vanishing derivative. Setting the first derivative to 0 produces
\begin{equation}
\begin{split}
2(k_1-k_0) s\mp\frac{k_2 s^2 +k_3 s}{\sqrt{1-s^2}}\pm k_2\sqrt{1-s^2}+k_4&=0\\
\Rightarrow (2(k_1-k_0)s+k_4)^2(1-s^2)\\
-(k_2 -2k_2s^2-k_3 s)^2&=0
\end{split}
\end{equation}
The reformulation is given by multiplying with $\pm\sqrt{1-s^2}$ and applying the third binomial formula. The resulting formula is the quartic equation
\begin{equation}
\begin{split}
-4((k_1-k_0)^2+k_2^2)s^4-4(k_4(k_1-k_0)+k_2k_3)s^3&\\
+(4(k_1-k_0)^2+4k_2^2-k_4^2-k_3^2)s^2&\\+2(2k_4(k_1-k_0)+k_2k_3)s+(k_4^2-k_2^2)=&0,
\end{split}
\label{eq:polynom}
\end{equation}
for which closed-form solution formulas exist \cite{jimenez2013fast} or Newton type methods can be applied. If the global maximum is not unique, the candidate with the smallest value for $|s|$ must be selected. Once $\hat s$ is found, so is $|\hat{c}|=\sqrt{1-\hat{s}^2}$. The sign of $\hat{c}$ is finally determined by evaluating $\rho(\Theta,\mathbf{G}_{p,q}(c,\hat{s})\cdot\tilde{\Theta})$ for $c=|\hat{c}|$ and $c=-|\hat{c}|$.

While we need the optimizer on $O(n)$, Algorithm \ref{alg:alpaca} searches only one of its connected components. Thus, a necessary condition for the solution of \eqref{eq:d_o} is that the determinant of the determinant of the initialization $\mathbf{Q}$ has the correct sign. Practically, this implies that the algorithm needs to be run at least twice with initializations both in $SO(n)$ and $O(n)\setminus SO(n)$.

\subsection{Convergence Properties}
	The purpose of this subsection is to take a closer a  look at Algorithm \ref{alg:alpaca} from the point of view of convergence and numerical complexity. For a discussion of convergence properties of Jacobi-type methods in general, the reader is referred to \cite{kleinsteuber2004jacobi}. The optimization problem \eqref{eq:d_o} is non-convex due to the restriction of $\mathbf{Q}$ to be in $O(n)$. Although Jacobi methods perform reasonably well in practice, unfortunately we lack a proof of global convergence for this particular method.

Let
\begin{equation}
	f(\mathbf{Q}^{(l)})=\rho(\Theta_1,\mathbf{Q}^{(l)}\cdot\Theta_2)
\end{equation}
be the cost function value  at iteration $l$ of Algorithm \ref{alg:alpaca} for any initialization. The sequence $(f(\mathbf{Q}^{(l)}))_{l\in\mathbb{N}}$ is bounded above and increases monotonically. Hence, it converges to a limit $\hat{f}$. In other words, the algorithm always terminates. The execution time is dictated by the solution of \eqref{eq:polynom}, for which fast and robust methods exist but which has to be performed $(n^2-n)/2$ times per iteration.
 
An essential question is, what the algorithm returns. Since we can not guarantee that it always returns a global maximum, we want at least to make sure that the result is (close to) a critical point of $f$. This is shown in Theorem \ref{thm:accu}. One claim that is made is that $f$ can be further increased at a non-critical point $\mathbf{Q}$ on $O(n)$ by multiplying $\mathbf{Q}$ with a Givens rotation. The claim follows from the observation that for any $\mathbf{Q}\in O(n)$, the set
 \begin{equation}
\left\{\frac{\partial}{\partial s}\mathbf{Q}\mathbf{G}_{p,q}(c,s)\Bigr|_{s=0}\right\}_{p\in{1,\dots,n-1},q\in{p+1,\dots,n}}
 \end{equation}
spans the \emph{tangent space} of $O(n)$ at $\mathbf{Q}$ which consists of products of $\mathbf{Q}$ with skew-symmetric matrices\cite{absil2009optimization}.

\begin{theorem}
		Let $(\mathbf{Q}^{(l)})_{l\in\mathbb{N}}$ be the sequence of orthogonal matrices generated by the outer while loop of Algorithm \ref{alg:alpaca}. The sequence $(\mathbf{Q}^{(l)})_{l\in\mathbb{N}}$ is bounded with respect to the Frobenius norm. If the order of each sweep is chosen randomly, each accumulation point is almost certainly a critical point of the cost function $f$.
		\label{thm:accu}
	\end{theorem}
	\begin{proof}
		Since $O(n)$ is bounded w.r.t the Frobenius norm, so is $(\mathbf{Q}^{(l)})_{l\in\mathbb{N}}$.
		
		Let $\hat{f}$ be the limit of the cost function value sequence $(f(\mathbf{Q}^{(l)}))_{l\in\mathbb{N}}$ and $o(l)$ the monotonically increasing index mapping belonging to a convergent subsequence of $(\mathbf{Q}^{(l)})_{l\in\mathbb{N}}$. The sequence $(\mathbf{Q}^{(o(l))})_{l\in\mathbb{N}}$ converges to a limit point $\hat{\mathbf{Q}}\in{O(n)}$ with $f(\hat{\mathbf{Q}})=\hat{f}$. Assume that $\hat{\mathbf{Q}}$ is not a critical point of $f$. This implies, that we can find a Givens rotation $\mathbf{G}_{p,q}(c,s)$, such that
			\begin{equation}
			f(\mathbf{G}_{p,q}(c,s)\hat{\mathbf{Q}})>f(\hat{\mathbf{Q}})=\hat{f}.
			\end{equation}
			Since the sequence $(\|\mathbf{Q}^{(o(l))}-\hat{\mathbf{Q}}\|_F)_{l\in\mathbb{N}}$ converges to 0, so does the sequence\newline $(\|\mathbf{G}_{p,q}(c,s)\mathbf{Q}^{(o(l))}-\mathbf{G}_{p,q}(c,s)\hat{\mathbf{Q}}\|_F)_{l\in\mathbb{N}}$, because the Frobenius norm is invariant under orthogonal transformations. The mapping $f$ is continuuous w.r.t. the Frobenius norm and thus there is an index $\eta \in \mathbb{N}$ such that
			\begin{equation}
			f(\mathbf{G}_{p,q}(c,s)\mathbf{Q}^{(o(l))})>\hat{f}
			\end{equation}
			holds for all $l\geq \eta$. However, at each iteration the algorithm starts the sweep by performing a Givens rotation. If the affected indexes for this first Givens rotation are chosen randomly, the algorithm almost certainly will eventually choose $p,q$ and thus yield the point $\mathbf{G}_{p,q}(\hat{c},\hat{s})\mathbf{Q}^{(o(\hat{l}))}$ with $\hat{l}\geq\eta$, where $\hat{c}$ and $\hat{s}$ are chosen such that $f$ is maximized, at the beginning of a sweep. For such a point, the inequality
			\begin{equation}
			\begin{split}
			f(\mathbf{G}_{p,q}(\hat{c},\hat{s})\mathbf{Q}^{(o(\hat{l}))})\geq f(\mathbf{G}_{p,q}(c,s)\mathbf{Q}^{(o(\hat{l}))})>\hat{f}
			\end{split}
			\end{equation}
			holds. This is not possible, because $f(\mathbf{Q}^{(o(\hat{l})+1)})\geq f(\mathbf{G}_{p,q}(\hat{c},\hat{s})\mathbf{Q}^{(o(\hat{l}))})$ and $\hat{f}$ is an upper bound for any element of $(f(\mathbf{Q}^{(l)}))_{l\in\mathbb{N}}$.
		
	\end{proof}

\subsection{Fr{\'e}chet Means of sets of KLDSs}
\label{sub:avg}
\begin{figure}
	\begin{center}
		\begin{tikzpicture}
		%	\draw[step=1.0,lightgray,thin] (-4,-4) grid (4,4);
		\filldraw[black] (1.16,-0.22) circle (2pt);
		\filldraw[black] (-0.40,0.36) circle (2pt);
		\filldraw[black] (-1.03,0.57) circle (2pt);
		\filldraw[black] (0.73,-0.85) circle (2pt);
		\filldraw[black] (-0.1,-0.5) circle (2pt);
		\filldraw[black] (0.1,0.5) circle (2pt);
		\draw[color=gray] (-3,-1) -- (0,0);
		\draw[color=gray,dashed] (-3,-1) -- (-0.46,0.14);
		\draw[color=gray,dotted] (0,0) -- (-0.46,0.14);
		\filldraw[blue] (-3,-1) circle (2pt);
		\filldraw[black] (-0.46,0.14) circle (2pt);
		\filldraw[red] (0,0) circle (2pt);
		\end{tikzpicture}
	\end{center}
	\caption{Motivation for using Fr\'echet means: For a set of reference points (black), the representative point (red) should be chosen such that it its distance to each one of them (dotted) is small. Then, the distance of a test point (blue) to the representative point (continuous) differs little from its distance to any of the reference points represented by the representative point (dashed).}
	\label{fig:triag_interpretation}
\end{figure}
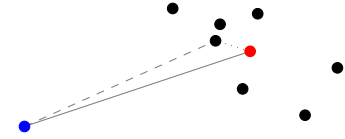

Due to the metric property, the alignment distance allows for the employment Fr\'echet means on subsets of $\mathcal{K}_{n,p,\kappa}$  \cite{afsari2012group}. Consider a finite subset $\{\boldsymbol{x}_1,\dots,\boldsymbol{x}_K\}$ of a space $\mathcal{X}$ equipped with a metric $d(\cdot,\cdot)$. The Fr\'echet mean \cite{frechet1948les} of $\{\boldsymbol{x}_1,\dots,\boldsymbol{x}_K\}$ is the minimizer
\begin{equation}
\bar{\boldsymbol{x}}=\argmin_{\boldsymbol{x}\in\mathcal{X}}\frac{1}{K}\sum_{i=1}^{K}d(\boldsymbol{x},\boldsymbol{x}_i)^2.
\end{equation}
This may look like an abstract concept at first, but it is a natural way to choose a representative point out of a set of candidates. We want this representative point not to be far away of any of the reference points, which is why the squared sum of metrics to all these reference points is minimized. By doing so, we make sure through the triangle inequality, that the metric from a test sample to the representative point differs as little as possible from the metric from a test sample to any reference point. Fig. \ref{fig:triag_interpretation} visualizes this intuition: if we consider a reference point $\boldsymbol{x}_{\text{ref}}$, a testing point $\boldsymbol{x}_{\text{test}}$ and a representative point $\boldsymbol{x}_{\text{rep}}$, the triangle inequality yields
	\begin{equation}
	\begin{split}
&d(\boldsymbol{x}_{\text{ref}},\boldsymbol{x}_\text{rep})-d(\boldsymbol{x}_{\text{rep}},\boldsymbol{x}_\text{test})\\
	\leq&  d(\boldsymbol{x}_{\text{ref}},\boldsymbol{x}_\text{test})
	\leq  d(\boldsymbol{x}_{\text{ref}},\boldsymbol{x}_\text{rep})+d(\boldsymbol{x}_{\text{rep}},\boldsymbol{x}_\text{test}).
	\end{split}
	\end{equation}	
	A small value for $d(\boldsymbol{x}_{\text{rep}},\boldsymbol{x}_\text{ref})$ thus makes sure that $\boldsymbol{x}_{\text{rep}}$ is a sensible approximation of $\boldsymbol{x}_\text{ref}$ for proximity based classification tasks.

Consider the  finite set $\{\Theta_i=(\mathbf{A}_i,\mathbf{Y}_i,\boldsymbol{\alpha}_i,\boldsymbol{\beta}_i)\}_{i\in\{1,\dots,K\}}\subset\mathcal{K}_{n,p,\kappa}$. The Fr\'echet mean of $\{\Theta_i\}_{i\in\{1,\dots,K\}}$ is the minimizer $\bar{\Theta}=(\bar{\mathbf{A}}_i,\bar{\mathbf{Y}}_i,\bar{\boldsymbol{\alpha}}_i,\bar{\boldsymbol{\beta}}_i)\in\mathcal{K}_{n,p,\kappa}$ of the average of alignment distances to the $K$ KLDSs:
\begin{equation}
	\begin{split}
		\bar{\Theta}=&\argmin_{\Theta\in\mathcal{K}_{n,p,\kappa}}\frac{1}{K}\sum_{i=1}^{K}d_\mathcal{O}(\Theta,\Theta_i)^2\\
		=& \argmin_{\Theta\in\mathcal{K}_{n,p,\kappa}}\sum_{i=1}^{K}\min_{\mathbf{Q}\in O(n)} d_F(\Theta,\mathbf{Q}\cdot\Theta_i)^2.
	\end{split}
	\label{eq:klds_avg}
\end{equation}
The minimization problem \eqref{eq:klds_avg} can be approached iteratively, until the cost function
\begin{equation}
g(\bar{\Theta})=\frac{1}{K}\sum_{i=1}^{K}d_\mathcal{O}(\bar{\Theta},\Theta_i)^2
\end{equation}
can not be further reduced.

Assume that at a given iteration $l$, an approximate solution $\bar{\Theta}^{(l)}$ was determined and let $\mathbf{Q}_i^{(l)}$ be the minimizer of $d_F(\bar{\Theta}^{(l)},\mathbf{Q}\cdot\Theta_i)^2$, i.e. the maximizer of $\rho(\bar{\Theta}^{(l)},\mathbf{Q}\cdot\Theta_i)^2$, for each $i$. Substituting this into \eqref{eq:klds_avg} yields
\begin{equation}
	\begin{split}
		\bar{\Theta}^{(l+1)}=&\argmin_{\Theta\in\mathcal{K}_{n,p,\kappa}}\sum_{i=1}^{K}  d_F(\Theta,\mathbf{Q}_i^{(l)}\cdot\Theta_i)^2\\
		=&\argmin_{\substack{\mathbf{A},\mathbf{Y},\boldsymbol{\alpha},\boldsymbol{\beta}\\\in\mathcal{K}_{n,p,\kappa}}}\lambda_A\sum_{i=1}^{K}\mathrm{tr}(\mathbf{A}^\top(\mathbf{A}-2\mathbf{Q}_i^{(l)\top}\mathbf{A}_i\mathbf{Q}_i^{(l)}))\\
		&-2\sum_{i=1}^{K}\mathrm{tr}(\boldsymbol{\alpha}^\top\kappa(\mathbf{Y},\mathbf{Y}_i)\boldsymbol{\alpha}_i\mathbf{Q}_i^{(l)})\\
		&+\lambda_\mu\sum_{i=1}^{K}(\boldsymbol{\beta}^\top\kappa(\mathbf{Y},\mathbf{Y})\boldsymbol{\beta}-2\boldsymbol{\beta}^\top\kappa(\mathbf{Y},\mathbf{Y}_i)\boldsymbol{\beta}_i).
	\end{split} 
	\label{eq:fopt}
\end{equation}
This minimization problem can be solved separately for $\mathbf{A}$, and the rest of the parameters. With this in mind, let us split up the problem. For $\mathbf{A}$, this yields
\begin{equation}
\bar{\mathbf{A}}^{(l+1)}
%&=&\argmin_{\substack{\mathbf{A}\in\mathbb{R}^{n\times n},\\\|\mathbf{A}\|_2<1}}\sum_{i=1}^{K}\mathrm{tr}(\mathbf{A}^\top(\mathbf{A}-2\mathbf{Q}_i^{(l)\top}\mathbf{A}_i\mathbf{Q}_i^{(l)})) \nonumber\\
=\argmin_{\substack{\mathbf{A}\in\mathbb{R}^{n\times n},\\\|\mathbf{A}\|_2<1}}\mathrm{tr}(\mathbf{A}^\top(K\mathbf{A}-2\sum_{i=1}^{K}\mathbf{Q}_i^{(l)\top}\mathbf{A}_i\mathbf{Q}_i^{(l)})).\label{eq:tropt1}
\end{equation}

For the remaining parameters, note that both the feature space bias as well as the feature space observer can be represented as a linear combination of the feature space mappings of all the involved samples. We could thus choose  $\mathbf{\bar{\mathbf{Y}}^{(l)}}$ to be fixed for all iterations $l$ as
\begin{equation}
\mathbf{Y}^*=\begin{bmatrix}\mathbf{Y}_1\cdots\mathbf{Y}_K\end{bmatrix}\in\mathbb{R}^{p\times N^*}
\label{eq:Y_avg}
\end{equation}
to make sure that the feature space bias and the feature space observer are determined exactly.
However, in practice a sample matrix of size $N^*=\sum_{i=1}^{K}N_i$ can become quickly not handleable. Since the feature space bias and the feature space observer typically operate on a much smaller dimension, it is reasonable to assume that fewer samples are needed to model them, i.e. a sample matrix with a considerably lower number of columns $\bar{N}$ can be employed. The choice of $\bar{\mathbf{Y}}\in\mathbb{R}^{p\times\bar{N}}$ can be made heuristically, following insights of Nystr\"om interpolation of kernel matrices \cite{drineas2005nystrom}. It would go beyond the scope of this work to review different strategies for the choice of $\bar{\mathbf{Y}}\in\mathbb{R}^{p\times\bar{N}}$. In this work, we employ the k-means approach proposed in \cite{zhang2008improved} and assume that $\kappa(\bar{\mathbf{Y}},\bar{\mathbf{Y}})$ has full rank.

Assuming that $\bar{\mathbf{Y}}\in\mathbb{R}^{p\times\bar{N}}$ is determined, let us fix
\begin{equation}
\boldsymbol{a}^{(l)}=\begin{bmatrix}\mathbf{Q}^{(l)\top}_1\boldsymbol{\alpha}_1^\top\dots\mathbf{Q}^{(l)\top}_K\boldsymbol{\alpha}_K^\top\end{bmatrix}^\top.
\label{eq:a_alpha}
\end{equation}
for each iteration $l$ and
\begin{equation}
\boldsymbol{b}=\begin{bmatrix}
	\boldsymbol{\beta}_1^\top \dots \boldsymbol{\beta}_K^\top
\end{bmatrix}^\top
\label{eq:b_beta}
\end{equation}
for all iterations. This leads to the formulations
\begin{eqnarray}
\bar{\boldsymbol{\alpha}}^{(l+1)}
%&=&\argmax_{\substack{\boldsymbol{\alpha}\in\mathbb{R}^{\bar{N}\times n},\\
%\boldsymbol{\alpha}^\top\kappa(\bar{\mathbf{Y}},\bar{\mathbf{Y}})\boldsymbol{\alpha}=\mathbf{I}_n }}\sum_{i=1}^{K}\mathrm{tr}(\boldsymbol{\alpha}^\top\kappa(\bar{\mathbf{Y}},\mathbf{Y}_i)\boldsymbol{\alpha}_i\mathbf{Q}_i^{(l)})\nonumber\\
&=&\argmax_{\substack{\boldsymbol{\alpha}\in\mathbb{R}^{\bar{N}\times n},\\
		\boldsymbol{\alpha}^\top\kappa(\bar{\mathbf{Y}},\bar{\mathbf{Y}})\boldsymbol{\alpha}=\mathbf{I}_n }}\mathrm{tr}(\boldsymbol{\alpha}^\top\kappa(\bar{\mathbf{Y}},\mathbf{Y}^*)\boldsymbol{a}^{(l)}),\label{eq:tropt3}\\
\bar{\boldsymbol{\beta}}^{(l+1)}
%&=&\argmin_{\boldsymbol{\beta}\in\mathbb{R}^{\bar{N}}} \sum_{i=1}^{K}\boldsymbol{\beta}^\top(\kappa(\bar{\mathbf{Y}},\bar{\mathbf{Y}})\boldsymbol{\beta}-2\kappa(\bar{\mathbf{Y}},\mathbf{Y}_i)\boldsymbol{\beta}_i)\nonumber\\
&=&\argmin_{\boldsymbol{\beta}\in\mathbb{R}^{\bar{N}}} \boldsymbol{\beta}^\top(K\kappa(\bar{\mathbf{Y}},\bar{\mathbf{Y}})\boldsymbol{\beta}-2\kappa(\bar{\mathbf{Y}},\mathbf{Y}^*)\boldsymbol{b}).\;\;\;\;\; \label{eq:tropt4}
\end{eqnarray}
The Equation \eqref{eq:tropt4} is an unconstrained quadratic minimization problem with the analytical solution
\begin{equation}
\bar{\boldsymbol{\beta}}=\frac{1}{K}\kappa(\bar{\mathbf{Y}},\bar{\mathbf{Y}})^{-1}\kappa(\bar{\mathbf{Y}},\mathbf{Y}^*)\boldsymbol{b}
\label{eq:beta_avg}
\end{equation}
for all iterations. Additionally, the multiplication from left or right with orthogonal matrices does not affect the spectral norm $\|\cdot\|_2$ of a matrix. Together with the triangle inequality, this unfolds that the solution of \eqref{eq:tropt1}, yielded by the euclidean average, 
\begin{equation}
	\bar{\mathbf{A}}^{(l)}=\frac{1}{K}\sum_{i=1}^{K}\mathbf{Q}_i^{(l)\top}\mathbf{A}_i\mathbf{Q}_i^{(l)},
	\label{eq:a_avg}
\end{equation}
is unaffected by the constraint $\|\mathbf{A}\|_2<1$. For the solution of \eqref{eq:tropt3}, the EVD of the symmetric kernel matrix is written as
\begin{equation}
	\kappa(\bar{\mathbf{Y}},\bar{\mathbf{Y}})=\mathbf{V}\mathbf{\Lambda} \mathbf{V}^\top.
	\label{eq:K_evd}
\end{equation}
From the constraint $\boldsymbol{\alpha}^\top\kappa(\bar{\mathbf{Y}},\bar{\mathbf{Y}})\boldsymbol{\alpha}=\mathbf{I}_n$, we can conclude that the solution must be of the form $\boldsymbol{\alpha}=\mathbf{V}\mathbf{\Lambda}^{-\frac{1}{2}}\mathbf{Z}'$, where  $\mathbf{Z}'^\top\mathbf{Z}'=\mathbf{I}_n$. Substituting this into \eqref{eq:tropt3} reduces the problem to finding
\begin{equation}
\mathbf{Z}'=\argmax_{\mathbf{Z}\in \mathbb{R}^{\bar{N}\times n},\ \mathbf{Z}^\top\mathbf{Z}=\mathbf{I}_n}\mathrm{tr}(\mathbf{Z}^\top\mathbf{\Lambda}^{-\frac{1}{2}}\mathbf{V}^\top\kappa(\bar{\mathbf{Y}},\bar{\mathbf{Y}})\boldsymbol{a}^{(l)}),
\label{eq:trpr}
\end{equation}
which can be solved as follows. Let the singular value decomposition (SVD) of the known factor of the trace product in \eqref{eq:trpr} be given by
\begin{equation}
	\mathbf{\Lambda}^{-\frac{1}{2}}\mathbf{V}^\top\kappa(\bar{\mathbf{Y}},\mathbf{Y}^*)\boldsymbol{a}^{(l)}=\mathbf{U}'\mathbf{\Sigma}'\mathbf{V}'^\top.
	\label{eq:TILDE_svd}
\end{equation}
Substituting it into \eqref{eq:trpr} yields $\mathrm{tr}(\mathbf{Z}^\top\mathbf{U}'\mathbf{\Sigma}'\mathbf{V}'^\top)$ which is maximized by  $\mathbf{Z}'=\mathbf{U}'{\mathbf{V}'}_{:,1:n}^\top$. Finally, we arrive at
\begin{equation}
	\begin{split}
		\bar{\boldsymbol{\alpha}}^{(l)}=&\mathbf{V}\mathbf{\Lambda}^{-\frac{1}{2}}\mathbf{Z}'=\mathbf{V}\mathbf{\Lambda}^{-\frac{1}{2}}\mathbf{U}'{\mathbf{V}'}_{:,1:n}^\top.
	\end{split}
	\label{eq:projected_alpha}
\end{equation}
Algorithm \ref{alg:vicuna} summarizes the described procedure.
\begin{algorithm}
	\SetSideCommentRight
	\KwIn{Set of KLDS descriptors $\{\Theta_i=(\mathbf{A}_i,\mathbf{Y}_i,\boldsymbol{\alpha}_i,\boldsymbol{\beta}_i)\}_{i\in\{1,\dots,K\}}\subset\mathcal{K}_{n,p,\kappa}$, size of sample matrix $\bar{N}$}
	\KwOut{Average KLDS $\bar{\Theta}$}
	$\mathbf{Y}^*\leftarrow\begin{bmatrix}\mathbf{Y}_1\cdots\mathbf{Y}_K\end{bmatrix}$\tcc*{Eq. \eqref{eq:Y_avg}}
	$\bar{\mathbf{Y}}\leftarrow\mathrm{k\mathrm{-}means}(\mathbf{Y}^*,\bar{N})$;\\
	$\boldsymbol{b}\leftarrow\begin{bmatrix}\boldsymbol{\beta}_1^\top\cdots\boldsymbol{\beta}_K^\top\end{bmatrix}^\top$\tcc*{Eq. \eqref{eq:b_beta}}
	$\bar{\boldsymbol{\beta}}\leftarrow\frac{1}{K}\kappa(\bar{\mathbf{Y}},\bar{\mathbf{Y}})^{-1}\kappa(\bar{\mathbf{Y}},\mathbf{Y}^*)\boldsymbol{b}$\tcc*{Eq. \eqref{eq:beta_avg}}
	
	\For{$i=1$ \KwTo $K$}{
		$\mathbf{Q}_i\leftarrow\mathbf{I}_n$;
	}
	\While{$g(\bar{\Theta})$ not converged}{
		$\bar{\mathbf{A}}\leftarrow\frac{1}{N}\sum_{i}^{K}\mathbf{Q}_i^\top\mathbf{A}_i\mathbf{Q}_i$\tcc*{Eq. \eqref{eq:a_avg}}
		$\boldsymbol{a}\leftarrow\begin{bmatrix}\mathbf{Q}_1^{\top}\boldsymbol{\alpha}_1^\top \cdots  \mathbf{Q}_K^{\top}\boldsymbol{\alpha}_K^\top\end{bmatrix}^\top$\tcc*{Eq. \eqref{eq:a_alpha}}
		$\mathbf{V},\mathbf{\Lambda}\leftarrow \mathrm{EVD}(\boldsymbol a^\top\kappa(\bar{\mathbf{Y}},\bar{\mathbf{Y}})\boldsymbol{a})$\tcc*{Eq. \eqref{eq:K_evd}}
		$\mathbf{U}',\mathbf{\Sigma}',\mathbf{V}'\leftarrow\mathrm{SVD}(\mathbf{\Lambda}^{-\frac{1}{2}}\mathbf{V}^\top\kappa(\bar{\mathbf{Y}},\mathbf{Y}^*)\boldsymbol{a}^{(l)})$;\\
		$\bar{\boldsymbol{\alpha}}\leftarrow\mathbf{V}\mathbf{\Lambda}^{-\frac{1}{2}}\mathbf{U}'{\mathbf{V}'}_{:,1:n}^\top$\tcc*{Eq. \eqref{eq:projected_alpha}}
		$\bar{\Theta}\leftarrow(\bar{\mathbf{A}},\bar{\mathbf{Y}},\bar{\boldsymbol{\alpha}},\bar{\boldsymbol{\beta}})$;\\
		\For{$i=1$ \KwTo $K$}{
			$\mathbf{Q}_i\leftarrow\argmax_{\mathbf{Q}\in O(n)}\rho(\bar{\Theta},\mathbf{Q}\cdot\Theta_i)$;
		}
	}
	\caption{Fr\'echet mean computation}
	\label{alg:vicuna}
\end{algorithm}

Fr{\'e}chet means are in general not unique. Besides, we can not expect Algorithm \ref{alg:vicuna} to find a global minimum due to the non-convexity of the whole problem. It should thus not be thought of as a method for finding \emph{the} representative point of a finite subset of $\mathcal{K}_{n,p,\kappa}$, but rather for finding a point that is not far away of any of its elements. Note that the algorithm always terminates, since the cost function is non-negative and the algorithm iterations do not increase it.
\section{Experiments}
\label{sec:exp}
\subsection{Overview}
This section analyzes the performance of the presented alignment distance on dynamic texture and dynamic scene recognition tasks performed in Matlab on an Intel Core i5-2400 machine. To this end, we convert each frame to a histogram representation. LBP \cite{ojala1996comparative} and BoW \cite{fei2005bayesian}  are chosen, due to their successful applications to still-image texture and scene recognition respectively. The performance of the KLDS parameters in combination with the presented alignment distance was compared against KLDS parameters in combination with the Maximum Singular Value distance proposed in \cite{chaudhry2009histograms} and the Martin distance \cite{chan2007classifying}. In order to put the numbers into context, state-of-the art results from recent publications are included. Throughout all experiments, the alignment distance is computed via the Jacobi-type method which is initialized twice for each sign of the determinant of $\mathbf{Q}$ once. The initializations are determined by creating random sets of orthogonal matrices and choosing the element with the lowest cost function value within. The code for reproducing the NCC experiments will be made available on IEEE Xplore upon publication.

\subsection{Dynamic Textures}
\subsubsection{Database}
The DynTex database \cite{dyntex} is a collection of high-resolution RGB texture videos. Three different subsets of DynTex have been compiled and labeled for recognition benchmarking.
\begin{itemize}
	\item \emph{DynTex Alpha}: The Alpha subset is composed of 60 dynamic textures divided into 3 classes, each containing 20 videos of \emph{Sea}, \emph{Grass}, and \emph{Trees}, respectively.
	\item \emph{DynTex Beta}: The Beta dataset is composed of 162 dynamic textures divided into 10 classes. The classes with the number of samples indicated in brackets are \emph{Sea} (20), \emph{Vegetation} (20), \emph{Trees} (20), \emph{Flags} (20), \emph{Calm Water} (20), \emph{Fountains} (20), \emph{Smoke} (16), \emph{Escalator} (7), \emph{Traffic} (9) and \emph{Rotation} (10).
	\item \emph{DynTex Gamma}: The Gamma dataset is composed of 264 dynamic textures divided into 10 classes. The classes with the number of samples indicated in brackets are \emph{Flowers} (29), \emph{Sea} (38), \emph{Naked trees} (25), \emph{Foliage} (35), \emph{Escalator} (7), \emph{Calm water} (30), \emph{Flags} (31), \emph{Grass} (23), \emph{Traffic} (9) and \emph{Fountains} (37).
\end{itemize}
\begin{figure}
	\begin{center}
		\includegraphics[scale=0.6]{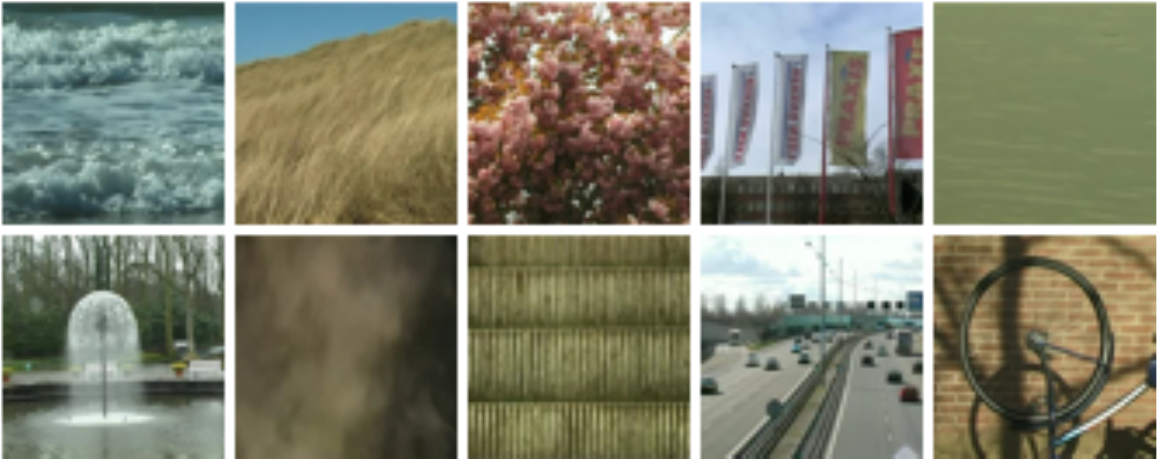}
	\end{center}
	\caption{DynTex Beta Video Frames}
	\label{fig:dyntex_beta}
\end{figure}
Fig. \ref{fig:dyntex_beta} depicts frames from the DynTex Beta Collection.
\subsubsection{Implementation Details}
The videos were converted to grayscale. An LBP histogram was computed for each frame via a third-party toolbox \cite{skarbnik2015}. The SoB KLDS parameters of order $n=5$ were computed by means of Algorithm \ref{alg:llama} from each stream of LBP histograms. Besides 1-NN performance, we are interested in the ability of Algorithm \ref{alg:vicuna} to represent sets of KLDSs via a representative point. To this end, we performed an NCC classification task additionally. Each class center is computed via Algorithm \ref{alg:vicuna}. In order to compare the performance to the Martin and the Maximum SV distance, the medoids were determined as class centers. For the NCC classification, $\lambda_\mu$ was set to $0$ for the sake of simplicity. Technically, this makes the alignment metric lose its positive definiteness on $\mathcal{K}_{n,p,\kappa}$. Practically, this is not a problem, since we simply work on an appropriate set projection of $\mathcal{K}_{n,p,\kappa}$. The parameter $\lambda_A$ was set to $0.25$ for all experiments, which produced the best results on the Alpha Dataset for NCC evaluation. For the 1-NN evaluation, two experiments have been performed to investigate the impact of the parameters. At first $\lambda_\mu$ was fixed to 0 and $\lambda_A$ varied in the range $0.15,0.35,\dots,1.95$. Then, $\lambda_A=0.25$ was kept and $\lambda_\mu$ was varied in the range $0,10,20,\dots,100$.
\subsubsection{Results}
Table \ref{tbl:dyntex_nn} shows the 1-NN classification results of SoB in comparison  to LBP-TOP \cite{zhao2007dynamic,qi2016dynamic}, aggregated salient features in three orthogonal planes (ASF-TOP) \cite{hong2016not} and \emph{Transferred Convolutional Net Features} (st-TCoF) \cite{qi2016dynamic}. The performance of SoB in combination with the aligned distance is plotted  against $\lambda_A$ in \ref{fig:lambda_A} and against $\lambda_\mu$ in Fig. \ref{fig:lambda_mu}.
\begin{table}
	\begin{center}
		\begin{tabular}{c || c | c | c}
			& Alpha & Beta & Gamma \\
			\hline \hline
			LBP-TOP  & 96.7\ \% & 85.8\ \% & 84.9\ \%\\ \hline
			ASF-TOP & 91.7\ \% & 86.4\ \% & 89.4\ \%\\ \hline
			st-TCoF & \textbf{98.3}\ \% & \textbf{98.2}\ \% & \textbf{98.1}\ \%\\ \hline
			\hline
			SoB + Martin & 91.7\ \% & 74.7\ \% & 64.4\ \%\\ \hline
			SoB + Max SV & 96.7\ \% & 84.0\ \% & 78.4\ \%\\ \hline
			SoB + Align & \textbf{98.3}\ \% & 90.1\ \% & 79.9\ \%\\ \hline
		\end{tabular}
	\end{center}
	\caption{Recognition rate on DynTex subsets: 1-NN}
	\label{tbl:dyntex_nn}
\end{table}
\begin{figure}
	\begin{center}
		\includegraphics[scale=0.35]{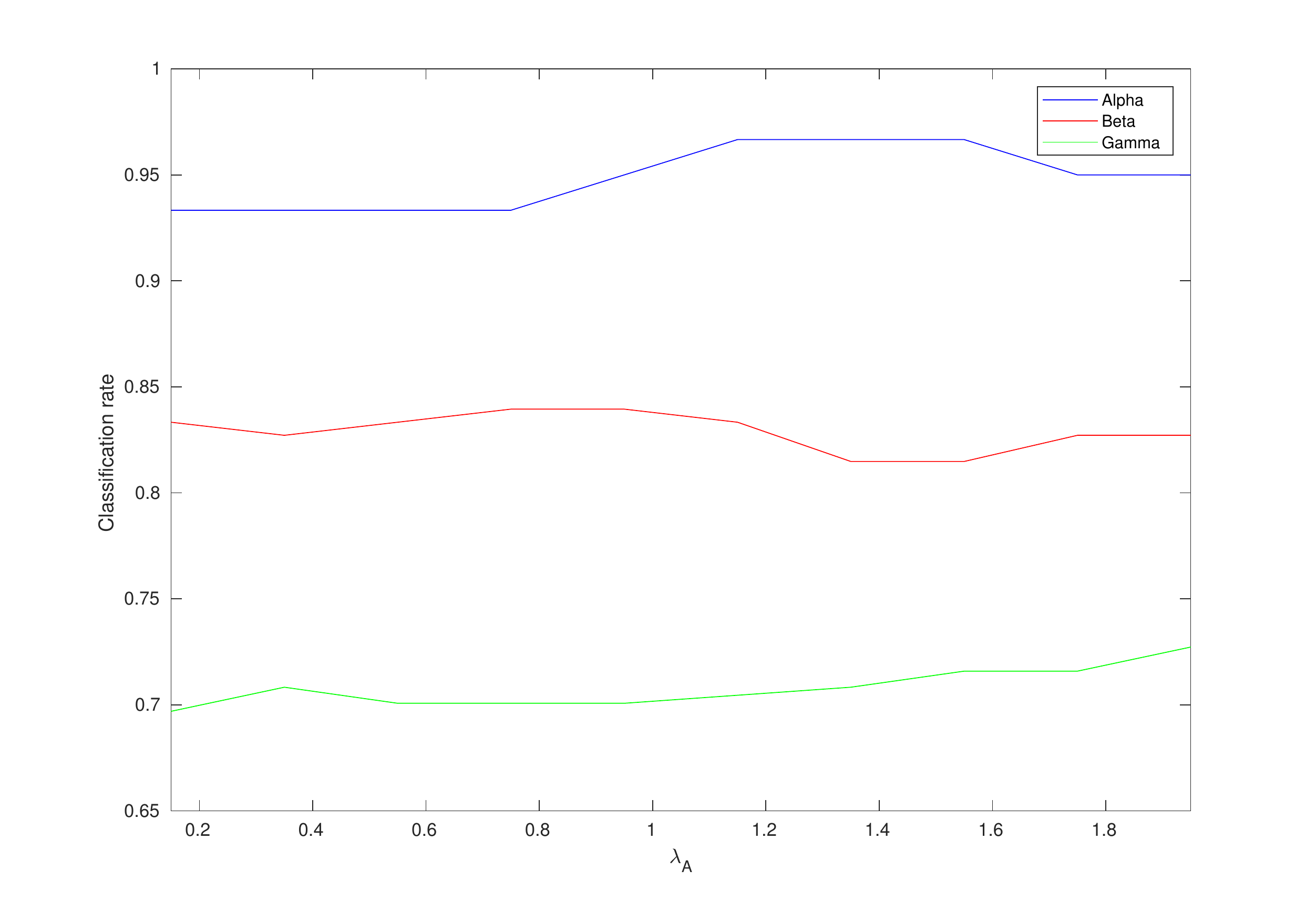}
	\caption{NN classification results of the DynTex datasets for different values of $\lambda_A$}
		\label{fig:lambda_A}
	\end{center}
\end{figure}
\begin{figure}
	\begin{center}
		\includegraphics[scale=0.35]{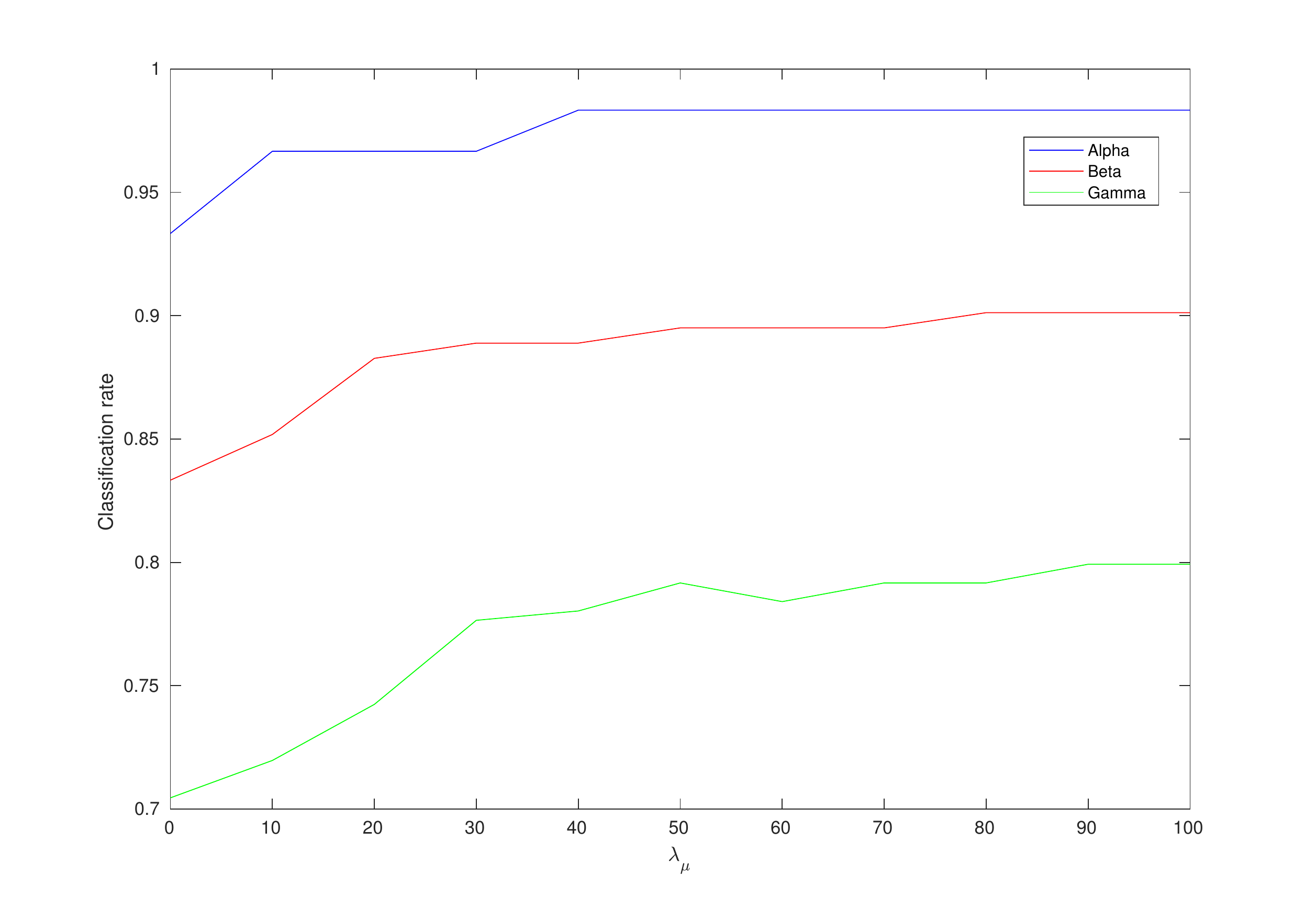}
		\caption{NN classification results of the DynTex datasets for different values of $\lambda_\mu$}
		\label{fig:lambda_mu}
	\end{center}
\end{figure}
SoB performs considerably better in combination with the alignment distance than in combination with the other two distance measures with a margin of at least 1.5 percentage points. Overall, it competes well with methods based on "shallow" representations, such as LBP-TOP or ASF-TOP outperforming them both on DynTex Alpha as well as on DynTex Beta. Still, the deep features learned by st-TCoF yield significantly better results. Table \ref{tbl:dyntex_ncc} shows the results of the NCC classification, which is considerably more challenging than the 1-NN classification and is more meaningful with regards to intra-class variability. For the Max SV and the Martin distance, the medoid was computed as an equivalent of the mean.  The alignment Fr\'echet mean of SoB sets provided significantly better result than the medoids with a margin of at least 3.3 percentage points. Beyond that, it yielded the best results published so far to our best knowledge, performing better than \emph{Dynamic Fractal Spectrum} (DFS) \cite{xu2011dynamic}, \emph{Spatiotemporal Curvelet Transform} (2D+T Curvelet) \cite{dubois2015characterization}, Orthogonal Tensor Dictionary Learning (OTDL), \cite{quan2015dynamic} on DynTex Alpha and Beta. Regarding the Gamma dataset, it should be noted that the authors of \cite{quan2015dynamic} wrongly describe it as a 10-class database of 275 samples, indicating that the experimental results for OTDL and DFS could have been produced on a different version from the one used in this work, while the evaluation of 2D+T Curvelet \cite{dubois2015characterization} was performed on an 11 class adaptation of this dataset, which is why these results are excluded.
\begin{table}
	\begin{center}
		\begin{tabular}{c || c | c | c}
			& Alpha & Beta & Gamma \\
			\hline \hline
			DFS  & 83.6\ \% & 65.2\ \% & 60.8\ \%\\ \hline
			2D+T Curvelet & 85.0\ \% & 67.0\ \% & -\\ \hline
			OTDL & 86.6\ \% & 69.0\ \% & 64.2\ \%\\ \hline
			\hline
			SoB + Martin & 83.3\ \% & 51.9\ \% & 41.7\ \%\\ \hline
			SoB + Max SV & 85.0\ \% & 59.9\ \% & 54.9\ \%\\ \hline
			SoB + Align & \textbf{88.3\ \%} &\textbf{ 75.3\ \%} & \textbf{67.1\  \%}\\ \hline
		\end{tabular}
	\end{center}
	\caption{Recognition Rate on DynTex subsets: NCC}
	\label{tbl:dyntex_ncc}
\end{table}
Table \ref{tbl:beta_confusion} shows the confusion matrix for NCC classification on DynTex Beta.
%The plot in Figure \ref{fig:lambda_mu} suggests that $\boldsymbol{\beta}$ is the most discriminative KLDS parameter. However, repeating the NCC experiment with a high $\lambda_\mu$ did not significantly improve the performance. A possible reason is the large variation in appearance within the classes of the DynTex datasets. While for 1-NN classification, for almost each video a counterpart with a similar still-image appearance, and thus a nearby feature space bias $\boldsymbol{\mu}$ can be found, still-image appearance is not sufficient to distinguish the classes in their entirety, and dynamics need to be taken into account. 
%\begin{table}
%	\renewcommand{\arraystretch}{1.5}
%	\begin{center}
%		\begin{tabular}{ r | c c c |}
%			& \smaller{Sea} & \smaller{Grass} & \smaller{Trees}\\
%			\hline
%			\smaller{Sea} & \smaller{\textbf{20}} &  &  \\
%			\smaller{Grass} & \smaller{2} &\smaller{\textbf{ 13}} & \smaller{5} \\
%			\smaller{Trees} &  &  & \smaller{\textbf{20}} \\
%			\hline
%		\end{tabular}
%	\end{center}
%	\caption{Confusion matrix for Dyntex Alpha (NCC Classification)}
%	\label{tbl:alpha_confusion}
%\end{table}
\begin{table}
	\renewcommand{\arraystretch}{1.25}
	\begin{center}
		\begin{tabular}{ r | c c c c c c c c c c}
			& \tiny{S} & \tiny{V} & \tiny{T} & \tiny{F} & \tiny{Cw} & \tiny{Fo} & \tiny{Sm} & \tiny{E} & \tiny{Tf} & \tiny{R}\\
			\hline
			\tiny{Sea} & \tiny{\textbf{18}} &  & & & \tiny{2} & & & & & \\
			\tiny{Vegetation} &  & \tiny{\textbf{15}} & \tiny{1} & & & \tiny{1} &  & \tiny{3} & \\
			\tiny{Trees} & &\tiny{2} & \textbf{\tiny{8}} & & & \tiny{3} & \tiny{3} & & \tiny{3} & \tiny{1} \\
			\tiny{Flags} & &  &  & \textbf{\tiny
				19} &  &  &  &  &  & \tiny{1}
				 \\
			\tiny{Calm water} & \tiny{1} &  &  &  & \textbf{\tiny{18}} &  & \tiny{1} & &  & \\
			\tiny{Fountain} & & \tiny{3} & \tiny{2} &  & \tiny{1} & \textbf{\tiny{11}} & \tiny{1} &  & \tiny{2} & \\
			\tiny{Smoke} &  &  &  & \tiny{4} &  &  & \textbf{\tiny{12}} &  & &  \\
			\tiny{Escalator} &  &  &  &  &  &  &  & \textbf{\tiny{6}} & \tiny{1} &  \\
			\tiny{Traffic} &  &  &  &  &  & & \tiny{1} &  & \textbf{\tiny{8}} &  \\
			\tiny{Rotation} &  &  &  &  &  & \tiny{2} & &  & \tiny{1} & \textbf{\tiny{7}} \\
		\end{tabular}
	\end{center}
	\caption{Confusion matrix for Dyntex Beta (NCC Classification)}
	\label{tbl:beta_confusion}
\end{table}
%\begin{table}
%	\renewcommand{\arraystretch}{1.25}
%	\begin{center}
%		\begin{tabular}{ r | c c c c c c c c c c}
%			& \tiny{F} & \tiny{S} & \tiny{Nt} & \tiny{Fol} & \tiny{E} & \tiny{Cw} & \tiny{F} & \tiny{G} & \tiny{Tf} & \tiny{Fo}\\
%			\hline
%			\tiny{Flowers} & \textbf{\tiny{16}} &  & \tiny{3} & \tiny{5} &  &  &  & \tiny{2} &  & \tiny{3} \\
%			\tiny{Sea} & & \textbf{\tiny{31}} &  & &  & \tiny{7} &  &  &  &  \\
%			\tiny{Naked trees} & \tiny{} &  & \textbf{\tiny{14}} & \tiny{6} &  & \tiny{1} & \tiny{3} &  & \tiny{1} & \\
%			\tiny{Foliage} & &  & \tiny{8} & \textbf{\tiny{24}} &  &  &  &  &  & \tiny{3} \\
%			\tiny{Escalator} & &  &  &  & \textbf{\tiny{6}} &  &  &  & \tiny{1} &  \\
%			\tiny{Calm water} &  & \tiny{4} &  & \tiny{1} &  & \textbf{\tiny{20}} & \tiny{3} &  &  & \tiny{2} \\
%			\tiny{Flags} &  & \tiny{1} & \tiny{1} &  & & & \textbf{\tiny{27}} &  &  & \tiny{2} \\
%			\tiny{Grass} &  &  & \tiny{2} & \tiny{3} & \tiny{3} & \tiny{1} &  & \textbf{\tiny{9}} & \tiny{2} & \tiny{3} \\
%			\tiny{Traffic} &  & &  &  &  &  & \tiny{1} &  & \textbf{\tiny{8}} &  \\
%			\tiny{Fountain} & \tiny{1} &  & \tiny{2} & \tiny{3} &  & \tiny{2} & \tiny{2} & \tiny{4} & \tiny{1} & \textbf{\tiny{22}} \\
%		\end{tabular}
%	\end{center}
%	\caption{Confusion matrix for Dyntex Gamma (NCC Classification)}
%	\label{tbl:gamma_confusion}
%\end{table}

\subsubsection{Runtime}
On average, the computation of the alignment distance for two KLDSs with state space dimension $n=5$ took 0.0316s. The overall runtimes for the NN (NCC) experiments for DynTex Alpha, Beta and Gamma with SoB + Align were  112.8s (539.47s), 828.98s (1,194.25s) and 2,205.2s (3,364.8s) respectively. The time used for feature computation was not considered.

\subsection{Dynamic Scenes}
\subsubsection{Database}
The YUPENN data set \cite{derpanis2012dynamic} is comprised of fourteen dynamic scene classes, \emph{Beach}, \emph{Elevator}, \emph{Forest fire}, \emph{Fountain},  \emph{Highway}, \emph{Lightning Storm}, \emph{Ocean}, \emph{Railway}, \emph{Rushing river}, \emph{Clouds}, \emph{Snowing}, \emph{City street}, \emph{Waterfalls} and \emph{Windmill farm}, each class containing 30 color videos of different durations and resolutions. Fig. \ref{fig:yupenn} depicts the different classes.
\begin{figure}
	\begin{center}
		\includegraphics[scale=0.6]{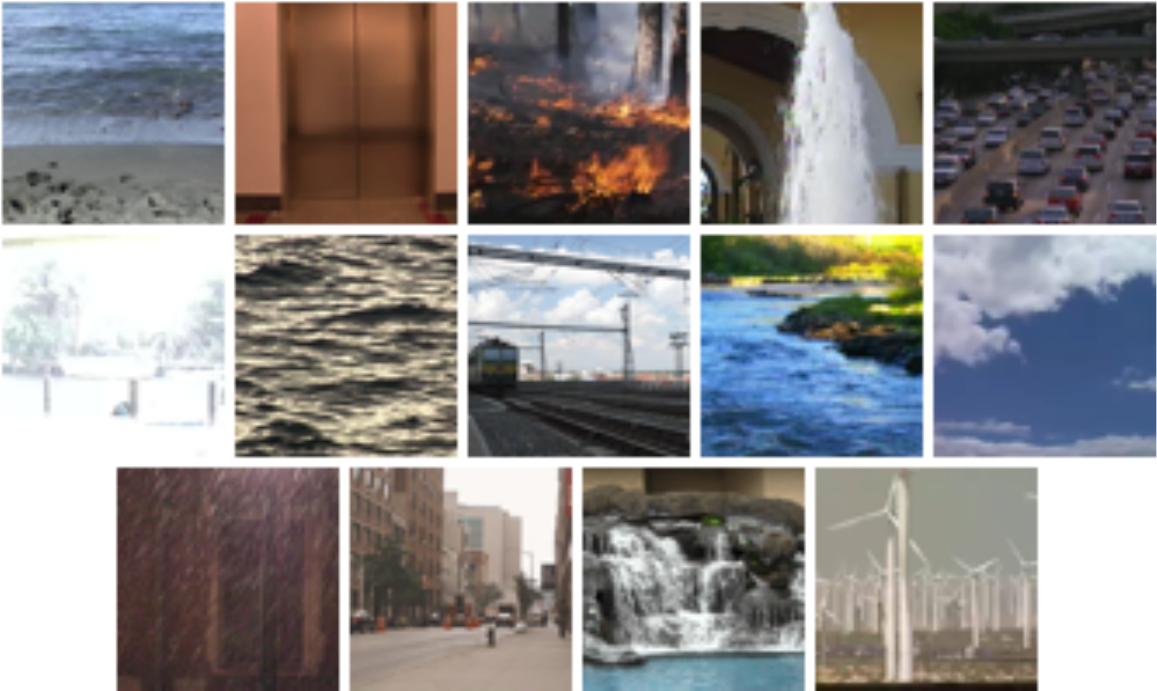}
	\end{center}
	\caption{YUPENN Database}
	\label{fig:yupenn}
\end{figure}
\subsubsection{Implementation Details}
OpenCV was used to compute the SURF features for the BoW model of size $p=500$.
For the 1-NN evaluation it needed to be made sure that the codebooks are not learned from the tested video samples. To this end, the dataset was divided into three equally large subsets and for each testing sample the codebook generated from the two subsets not containing the sample is used. That way each video sample in the dataset was evaluated by converting the dataset into a set of histogram streams with respect to one of three codebooks. The resulting classification rate was determined by averaging the classification rates for the three codebooks.
The histogram streams were converted to KLDSs of state space dimension $n=10$ which is chosen higher than for the previous experiment in order to account for the higher dimensionality of the feature vectors. Initially, $\lambda_A$ is set to $0.25$ and $\lambda_\mu$ to $0$, as was done for the NCC experiments on the DynTex database. The experiment was repeated for different values of $\lambda_A$ and $\lambda_\mu$.
\subsubsection{Results}
\begin{table}
	\begin{center}
		\begin{tabular}{c | c | c | c}
			Martin & Max SV & Align w/o tuning & Align best \\
			\hline
		    83.3 & 80.7 & 86.4 \% & 88.8 \%
		\end{tabular}
	\end{center}
	\caption{Recognition Rate on YUPENN dataset: SoB}
	\label{tbl:yupenn_sob}
\end{table}
\begin{table}
	\begin{center}
		\begin{tabular}{c | c | c | c }
			SOE & TSVQ & BoST & st-TCoF  \\
			\hline
			74\ \% & 69\ \% & 85\ \% & \ \textbf{98\ \%}
		\end{tabular}
	\end{center}
	\caption{Recognition Rate on YUPENN dataset: State of The Art}
	\label{tbl:yupenn_soa}
\end{table}
Table \ref{tbl:yupenn_sob} shows the classification result for SoB based approaches as described in this work. For the configuration that was used in the previous NCC experiments on the DynTex datasets, the alignment distance performs significantly better than the other two distance measure with 86.4 \% correct classifications. Of all tested parameters, $\lambda_\mu=0.8$ and $\lambda_A=0.5$ yielded the best results and could further improve the classification rate up to 88.8 \%. Table \ref{tbl:yupenn_soa} shows the classification results of recently proposed and well-received approaches: \emph{Spatiotemporal Oriented Energy} (SOE) \cite{derpanis2012dynamic}, \emph{Tree-Structured Vector Quantization (TSVQ)} \cite{nister2006scalable}, \emph{Bag of System Trees} (BoST) \cite{mumtaz2015scalable} and  \emph{Transferred Convolutional Net Features} (st-TCoF) \cite{qi2016dynamic}.
Like in the DynTex experiments, the st-TCoF descriptors outperform all other approaches in 1-NN classification with a considerable margin. These results are in line with the success of convolutional neural networks (CNNs) in the area of computer vision in the recent years. Hence, SoBs based on bags learned by CNNs could lead to a significant improvement in performance of the presented approach.
\subsubsection{Runtime}
On average, the computation of the alignment distance for two KLDSs with state space dimension $n=10$ took 0.1781s. The overall runtime for the experiment with non-tuned SoB + Align was 31,423s. The time used for feature computation was not considered.

\section{Conclusion}
Generative and statistical models are widely used in recently presented video and image descriptors. This paper discusses the modeling of videos as streams of histograms generated by a KLDS. As a framework for recognition and classification, this work presents a distance measure on KLDS parameters that allows for computing the dissimilarity of pairs and procrustean means of sets of visual processes described by temporally evolving histograms. The resulting framework competes well with state-of-the-art approaches on widely used dynamic scene and dynamic texture benchmarks. In particular, employing procrustean alignment means for NCC classification outperforms state-of the art approaches on the task of classifying dynamic textures.
\bibliographystyle{IEEEtran}


\begin{IEEEbiography}[{\includegraphics[width=1in,clip,keepaspectratio]{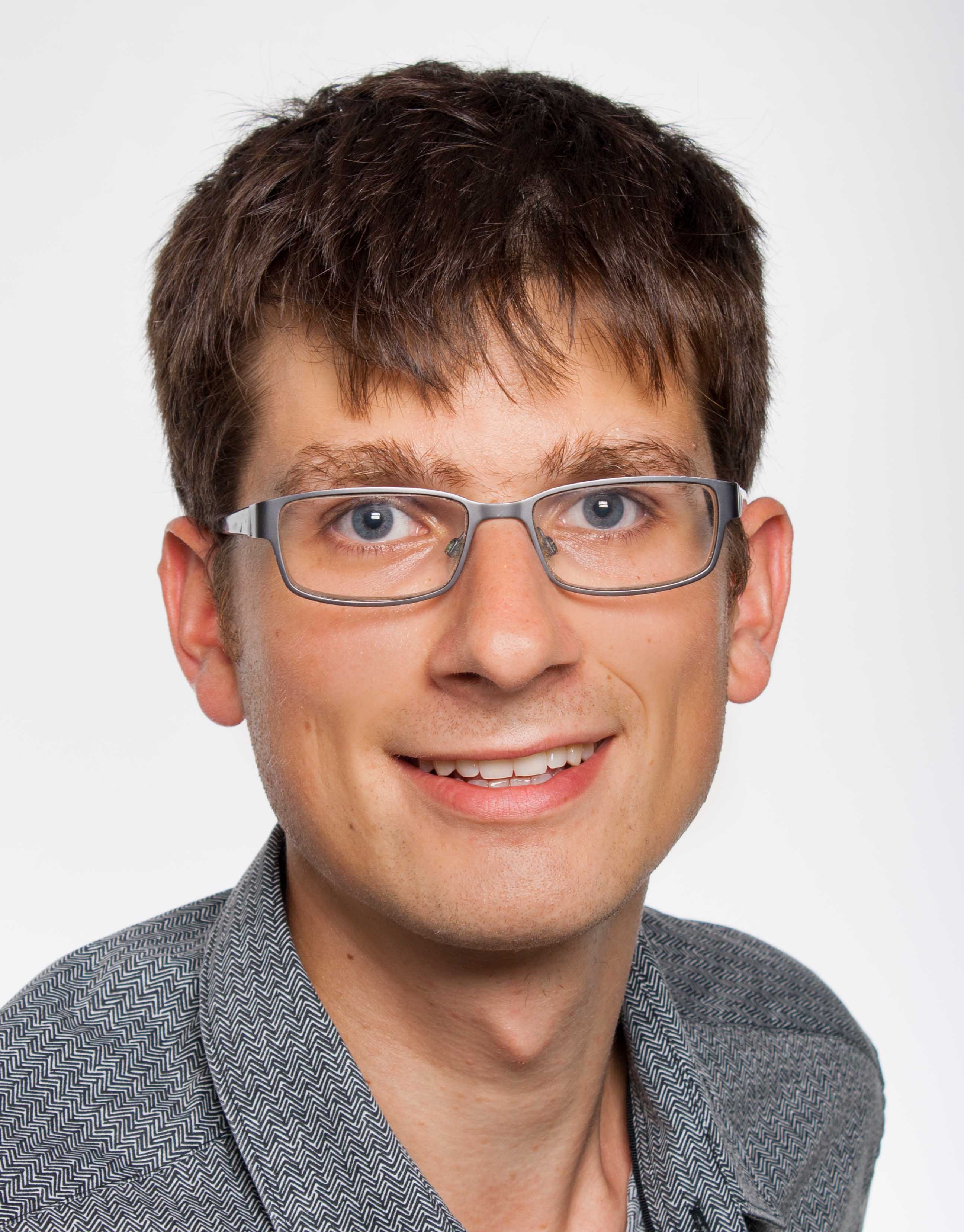}}]{Alexander Sagel} received his Bachelor's degree in Computer Engineering from the Technical University of Hamburg, Germany, and his Master's degree in Electrical Engineering from the Technical University of Munich, Germany. He is currently a doctoral candidate at the Chair for Data Processing within the Department of Electrical and Computer Engineering at the Technical University of Munich, Germany.
\end{IEEEbiography}
\begin{IEEEbiography}[{\includegraphics[ width=1in ,keepaspectratio,trim=0 1.4in 0 0.8in  ,clip]{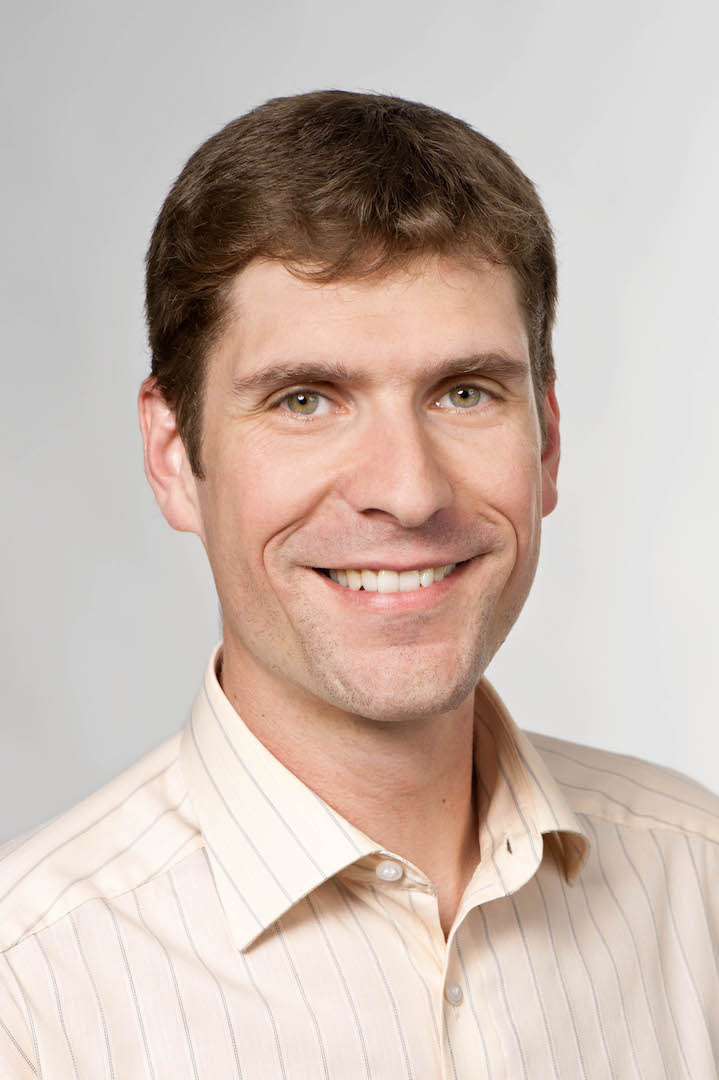}}]{Martin Kleinsteuber}
	received his Ph.D. in Mathematics from the University of W\"urzburg, Germany, in 2006. After post-doc positions at National ICT Australia Ltd., the Australian National University, Canberra, Australia, and the University of W\"urzburg, he has been appointed assistant professor for geometric optimization and machine learning at the Department of Electrical and Computer Engineering, TU Munich, Germany, in 2009. He won the SIAM student paper prize in 2004 and the Robert-Sauer-Award of the Bavarian Academy of Science in 2008 for his works on Jacobi-type methods on Lie algebras. Since 2016, he is leading the Data Science Group at Mercateo AG, Munich.
\end{IEEEbiography}
%\listoffigures
%\newpage
%\listoftables
\end{document}